\documentclass[twoside]{article}
\newif\iffullpaper
\fullpapertrue

\iffullpaper
\else
\usepackage[accepted]{aistats2020}
\fi

\usepackage[utf8]{inputenc} 
\usepackage[T1]{fontenc}    
\usepackage{url}            
\usepackage{booktabs}       
\usepackage{amsfonts}       
\usepackage{nicefrac}       
\usepackage{microtype}      
\usepackage{balance}
\usepackage{amsfonts}
\usepackage{amsmath}
\usepackage{amsthm}
\usepackage{amssymb}
\usepackage{comment}
\usepackage{dsfont}
\usepackage{color}
\usepackage{bm}
\usepackage{url}
\usepackage{thmtools}
\usepackage{thm-restate}
\usepackage{enumerate}
\usepackage[colorlinks,urlcolor=blue,citecolor=blue,linkcolor=blue]{hyperref}

\usepackage[round]{natbib}

\iffullpaper
\usepackage[margin=1in]{geometry}
\else
\fi

\usepackage{graphicx}

\newtheorem{theorem}{Theorem}[section]
\newtheorem{proposition}[theorem]{Proposition}
\newtheorem{lemma}[theorem]{Lemma}
\newtheorem{corollary}[theorem]{Corollary}
\newtheorem*{mainthm}{Main Theorem}

\theoremstyle{definition}

\theoremstyle{remark}

\newtheorem{rem}[theorem]{Remark}

\numberwithin{equation}{section}

\DeclareMathOperator*{\argmin}{argmin}

\newcommand{\RR}{\mathbb{R}}

\newcommand{\eps}{\varepsilon}
\newcommand{\vu}{\mathbf{u}}
\newcommand{\vw}{\mathbf{w}}

\newcommand{\vx}{\mathbf{x}}
\newcommand{\vb}{\mathbf{b}}

\newcommand{\dd}{\mathrm{d}}
\newcommand{\vz}{\mathbf{z}}
\newcommand{\vy}{\mathbf{y}}
\newcommand{\mB}{\mathbf{B}}
\newcommand{\mA}{\mathbf{A}}
\newcommand{\mI}{\mathbf{I}}
\newcommand{\vxi}{\bm{\xi}}
\newcommand{\vomega}{\bm{\omega}}

\newcommand{\KL}{\mathrm{KL}}
\newcommand{\TV}{\mathrm{TV}}

\newcommand{\innp}[1]{\left\langle #1 \right\rangle}

\newcommand{\lv}{\lVert}
\newcommand{\rv}{\rVert}

\renewcommand{\epsilon}{\varepsilon}

\begin{document}
\iffullpaper
\title{\textbf{Langevin Monte Carlo without smoothness}}
\author{
Niladri S. Chatterji\footnote{Equal contribution.}\phantom{$^\ast$}\\
University of California, Berkeley\\
\textsf{chatterji@berkeley.edu}\\
\and
Jelena Diakonikolas$^\ast$\\
University of Wisconsin, Madison\\
\textsf{jdiakonikola@wisc.edu}\\
\and
Michael I. Jordan\\
University of California, Berkeley\\
\textsf{jordan@cs.berkeley.edu}\\
\and
Peter L. Bartlett\\
University of California, Berkeley\\
\textsf{peter@berkeley.edu}\\
}
\fi
\maketitle

\begin{abstract}
Langevin Monte Carlo (LMC) is an iterative algorithm used to generate samples from a distribution that is known only up to a normalizing constant. The nonasymptotic dependence of its mixing time on the dimension and target accuracy is understood mainly in the setting of smooth (gradient-Lipschitz) log-densities, a serious limitation for applications in machine learning. In this paper, we remove this limitation, providing polynomial-time convergence guarantees for a variant of LMC in the setting of nonsmooth log-concave distributions. At a high level, our results follow by leveraging the implicit smoothing of the log-density that comes from a small  Gaussian perturbation that we add to the iterates of the algorithm and controlling the bias and variance that are induced by this perturbation.
\end{abstract}
\section{Introduction}\label{sec:intro}
The problem of generating a sample from a distribution that is known up to a normalizing constant is a core problem across the computational and inferential sciences~\citep{robert2013monte, kaipio2006statistical, cesa2006prediction, rademacher2008dispersion,  vempala2005geometric, chen2018fast}. A prototypical example involves generating a sample from a log-concave distribution---a probability distribution of the following form:
$$
p^*(\vx) \propto e^{-U(\vx)},
$$
where the function $U(\vx)$ is convex and is referred to as the \emph{potential function}. While generating a sample from the exact distribution $p^*(\vx)$ is often computationally intractable, for most applications it suffices to generate a sample from a distribution $\Tilde{p}(\vx)$ that is close to $p^*(\vx)$ in some distance (such as, e.g., total variation distance, Wasserstein distance, or Kullback-Leibler divergence).

The most commonly used methods for generating a sample from a log-concave distribution are (i) random walks
~\citep{dyer1991random,lovasz2007geometry}, (ii) different instantiations of Langevin Monte Carlo (LMC)~\citep{parisi1981correlation}, and (iii) Hamiltonian Monte Carlo (HMC)~\citep{neal2011mcmc}. 
These methods trade off rate of convergence against per-iteration complexity and applicability: random walks are typically the slowest in terms of the total number of iterations, but each step is fast as it does not require gradients of the log-density and they are broadly applicable, while HMC is the fastest in the number of iterations, but each step is slow as it uses gradients of the log-density and it mainly applies to distributions with smooth log-densities.

LMC occupies a middle ground between random walk and HMC.  In its standard form, LMC updates its iterates as:
\begin{equation}\label{eq:discrete-langevin} \tag{LMC}
    \vx_{k+1} = \vx_k - \eta \nabla U(\vx_k) + \sqrt{2\eta}\vxi_{k},
\end{equation}
where $\vxi_k \sim \mathcal{N}(\mathbf{0}, I_{d\times d})$ are independent Gaussian random vectors.  The per-iteration complexity is reduced relative to HMC because it only requires stochastic gradients of the log-density~\citep{welling2011bayesian}.  This also increases its range of applicability relative to HMC. While it is not a reversible Markov chain and classical theory of MCMC does not apply, it is nonetheless amenable to theoretical analysis given that it is obtained via discretization of an underlying stochastic differential equation (SDE).  There is, however, a fundamental difficulty in connecting theory to the promised wide range of applications in statistical inference. In particular, the use of techniques from SDEs generally requires $U(\vx)$ to have Lipschitz-continuous gradients. This assumption excludes many natural applications~\citep{kaipio2006statistical,durmus2018efficient,marie2019preconditioned,li2018graph}. 

%
A prototypical example of sampling problems with nonsmooth potentials are different instantiations of sparse Bayesian inference. In this setting, one wants to sample from the posterior distribution of the form:
\begin{align*}
    p^*(\vx) \propto \exp\big(-f(\vx)-\|\Phi \vx\|_p^p\big),
\end{align*}
where $f(\vx)$ is the log-likelihood function, $\Phi$ is a sparsifying dictionary (e.g., a wavelet dictionary), and $p \in [1, 2]$. In the simplest case of Bayesian LASSO~\citep{park2008bayesian}, $f(\vx) = \|\mA \vx - \vb\|_2^2$, $\Phi = \mI,$ and $p = 1,$ where $\mA$ is the measurement matrix, $\vb$ are the labels, and $\mI$ denotes the identity matrix. In general, when $\Phi$ is the identity or an orthogonal wavelet transform, proximal maps (i.e., solutions to convex minimization problems of the form $\min_{\vx \in \RR^d}\{\|\Phi\vx\|_p^p + \frac{1}{2\lambda}\|\vx - \vz\|_2^2\},$ where $\lambda$ and $\vz$ are parameters of the proximal map) are easily computable and proximal LMC methods apply~\citep{cai2018uncertainty,price2018sparse,durmus2019analysis,durmus2018efficient,atchade2015moreau}. However, in the so-called analysis-based approaches with overcomplete dictionaries, $\Phi$ is non-orthogonal and the existence of efficient proximal maps becomes unclear~\citep{elad2007analysis,cherkaoui2018analysis}. 


In this work, we tackle this problem head-on and pose the following question:
%
\begin{center}
   \emph{Is it possible to obtain nonasymptotic convergence results for LMC with a nonsmooth potential?}
\end{center}
Here, we focus on standard LMC (allowing only minor modifications) and the general case in which proximal maps are not efficiently computable. 
We answer this question positively through a series of results that involve transformations of the basic stochastic dynamics in~\eqref{eq:discrete-langevin}. 
In contrast to previous work that considered nonsmooth potentials \citep[e.g.,][]{atchade2015moreau,durmus2018efficient,hsieh2018mirrored,durmus2019analysis}, the transformations we consider are simple (such as perturbing a gradient query point by a Gaussian), they do not require strong assumptions such as the existence of proximal maps, they can apply directly to nonsmooth Lipschitz potentials without any additional structure (such as composite structure in~\citet{atchade2015moreau,durmus2018efficient} or strong convexity in~\citet{hsieh2018mirrored}), and the guarantee we provide is on the distribution of the last iterate of LMC as opposed to an average of distributions over a sequence of iterates of LMC in~\citet{durmus2019analysis}. 

Our main theorem is based on a Gaussian smoothing result  summarized in the following theorem.

\begin{mainthm}[Informal]
Let $\bar{p}^*(\vx) \propto \exp(-\bar{U}(\vx))$ be a probability distribution, where $\bar{U}(\vx) = U(\vx) + \psi(\vx),$ $U(\cdot)$ is a convex subdifferentiable function whose subgradients $\nabla U(\cdot)$ satisfy 
\iffullpaper
$$
(\exists L <\infty,\, \alpha \in [0, 1]):\; \|\nabla U(\vx) - \nabla U(\vy)\|_2 \leq L \|\vx - \vy\|_2^{\alpha}, \quad \forall \vx,\, \vy \in \RR^d,
$$
\fi
and $\psi(\cdot)$ is $\lambda$-strongly convex and $m$-smooth. There exists an algorithm---Perturbed Langevin Monte Carlo~\eqref{eq:modifed-LMC}---whose iterations have the same computational complexity as~\eqref{eq:discrete-langevin} and that requires no more than $\widetilde{\cal O}(d^{\frac{5-3\alpha}{2}}\big/\varepsilon^{\frac{4}{1+\alpha}})$ iterations to generate a sample that is $\varepsilon$-close to $\bar{p}^*$ in 
2-Wasserstein distance.

Further, if the goal is to sample from $p^*(\vx) \propto \exp(-U(\vx)),$ a variant of~\eqref{eq:modifed-LMC} takes 
{poly}$(d/\varepsilon)$ 
iterations to generate a sample from a distribution that is $\varepsilon$-close to $p^*$ in total variation distance.
\end{mainthm}
This informal version of the theorem displays only the dependence on the dimension $d$ and accuracy $\varepsilon$. A detailed statement is provided in Theorems~\ref{thm:g-smoothing-mixing-time-mLMC} and \ref{thm:TVcompositecontract}, and Corollary~\ref{cor:reg-potentials-result}. 

Our assumption on the subgradients of $U$ from the statement of the Main Theorem is known as H\"{o}lder-continuity, or $(L, \alpha)$-weak smoothness of the function.  It interpolates between Lipschitz gradients (smooth functions, when $\alpha =1$) and bounded gradients (nonsmooth Lipschitz functions, when $\alpha = 0$). In Bayesian inference, the general $(L, \alpha)$-weakly smooth potentials arise in the Bayesian analog of ``bridge regression,'' which interpolates between LASSO and ridge regression \citep[see, e.g.,][]{park2008bayesian} . To the best of our knowledge, our work is the first to consider the convergence of LMC in this general weakly-smooth model of the potentials -- previous work only considered its extreme cases obtained for $\alpha = 0$ and $\alpha = 1.$

To understand the behavior of LMC on weakly smooth (including nonsmooth) potentials, we leverage results from the optimization literature. First, by using the fact that a weakly smooth function can be approximated by a smooth function---a result that has been exploited in the optimization literature to obtain methods with optimal convergence rates~\citep{nesterov2015universal,devolder2014first}---we show that even the basic version of LMC can generate a sample in polynomial time, as long as $U$ is ``not too nonsmooth'' (namely, as long as $1/\alpha$ can be treated as a constant). 

The main impediment to the convergence analysis of LMC when treating a weakly smooth function $U$ as an inexact version of a nearby smooth function is that a constant bias is induced on the 
gradients, as discussed in Section~\ref{sec:main-deterministic}. To circumvent this issue, in Section~\ref{sec:main-gaussian} we argue that an LMC algorithm can be analyzed as a different LMC run on a Gaussian-smoothed version of the potential using \emph{unbiased stochastic estimates of the gradient}.\footnote{A similar idea was used in~\citet{kleinberg2018alternative} to view expected iterates of stochastic gradient descent as gradient descent on a smoothed version of the objective. Stochastic smoothing has also been used to lower the parallel complexity of nonsmooth minimization~\citep{duchi2012randomized}.} Building on this reduction, we define a Perturbed Langevin Monte Carlo~\eqref{eq:modifed-LMC} algorithm that reduces the additional variance that arises in the gradients from the reduction. 

To obtain our main theorem, we couple a result about convergence of LMC with stochastic gradient estimates in \emph{Wasserstein} distance~\citep{durmus2019analysis} with carefully combined applications of inequalities relating Kullback-Leibler divergence, Wasserstein distance, and total variation distance. Also useful are structural properties of the weakly smooth potentials and their Gaussian smoothing. As a byproduct of our techniques, we obtain a nonasymptotic result for convergence in \emph{total variation} distance for (standard) LMC with stochastic gradients, which, to the best of our knowledge, was not known prior to our work. 

\subsection{Related work} 
Starting with the work of Dalalyan~\citep{dalalyan2017theoretical}, a variety of theoretical results have established mixing time results for LMC~\citep{durmus2016high,raginsky2017non,zhang2017hitting,cheng2018convergence,cheng2018underdamped,dalalyan2019user,xu2018global,lee2018algorithmic} and closely related methods, such as Metropolis-Adjusted LMC~\citep{dwivedi2018log} and HMC~\citep{mangoubi2017rapid,bou2018coupling,mangoubi2018dimensionally,cheng2018sharp}. These results apply to sampling from well-behaved distributions whose potential function $U$ is \emph{smooth} (Lipschitz gradients) and (usually) strongly convex. For standard~\eqref{eq:discrete-langevin} with smooth and strongly convex potentials, the tightest upper bounds for the mixing time are $\widetilde{\mathcal{O}}(d/\epsilon^2)$.  They were obtained in \citet{dalalyan2017theoretical,durmus2016high} for convergence in total variation (with a \emph{warm start}; without a warm start the total variation result scales as $\widetilde{\mathcal{O}}(\frac{d^3}{\epsilon^2})$) and in 2-Wasserstein distance.  

When it comes to using~\eqref{eq:discrete-langevin} with nonsmooth potential functions, there are far fewer results. In particular, there are two main approaches: relying on the use of proximal maps~\citep{atchade2015moreau,durmus2018efficient,durmus2019analysis} and relying on averaging of the distributions over iterates of LMC~\citep[SSGLD]{durmus2019analysis}. Methods relying on the use of proximal maps require a composite structure of the potential (namely, that the potential is a sum of a smooth and a nonsmooth function) and that the proximal maps can be computed efficiently. 
Note that this is a very strong assumption. 
In fact, when the composite structure exists in convex optimization \emph{and} proximal maps are efficiently computable, it is possible to solve nonsmooth optimization problems with the  same iteration complexity as if the objective were smooth \citep[see, e.g.,][]{beck2009fast}. 
 Thus, while the methods from~\citet{durmus2018efficient,durmus2019analysis} have a lower iteration complexity than our approach
 , the use of proximal maps increases their per-iteration complexity (each iteration needs to solve a convex optimization problem). It is also unclear how the performance of the methods degrades when the proximal maps are computed only approximately. Finally, unlike our work,~\citet{atchade2015moreau,durmus2018efficient} and~\citet[SGLD]{durmus2019analysis} do not handle potentials that are purely nonsmooth, without a composite structure. 
 
 The only method that we are aware of and that is directly applicable to nonsmooth potentials is~\citep[SSGLD]{durmus2019analysis}. On a technical level,~\citet{durmus2019analysis} interprets LMC as a gradient flow in the space of measures and leverages techniques from convex optimization to analyze its convergence. The convergence guarantees are obtained for a weighted average of distributions of individual iterates of LMC, which, roughly speaking, maps the standard convergence analysis of the average iterate of projected gradient descent or stochastic gradient descent to the setting of sampling methods. While the iteration complexity for the average distribution~\citep{durmus2019analysis} is much lower than ours, their bounds for individual iterates of LMC are uninformative. By contrast, our results are for the \emph{last iterate} of perturbed LMC (P-LMC). Note that in the related setting of convex optimization, last-iterate convergence is generally more challenging to analyze and has been the subject of recent research~\citep{shamir2013stochastic,jain2019making}. 

It is also worth mentioning that there exist approaches such as the Mirrored Langevin Algorithm \citep{hsieh2018mirrored} that can be used to efficiently sample from  structured nonsmooth distributions such as the Dirichlet posterior. However, this algorithm's applicability to general nonsmooth densities is unclear.
\subsection{Outline}  
 Section~\ref{sec:prelims} provides the notation and background. 
 Section~\ref{sec:main} provides our main theorems, stated for deterministic and stochastic approximations of the potential (negative log-density) and composite structure of the potential. Section~\ref{sec:reg-potentials} extends the result of Section~\ref{sec:main} to non-composite potentials. We conclude in Section~\ref{sec:discussion}. 
\section{Preliminaries}\label{sec:prelims} 
The goal is to generate samples from a distribution $p^* \propto \exp(-U(\vx))$, where $\vx \in \mathbb{R}^d$. We equip 
$\RR^d$ with the standard Euclidean norm $\|\cdot\| = \|\cdot\|_2$ and use $\innp{\cdot, \cdot}$ to denote inner products. We assume the following for the potential (negative log-density) $U$:
\begin{enumerate}[({A}1)]
    \item \label{assumption:convex} $U$ is convex and subdifferentiable. Namely, for all $\vx \in \RR^d,$ there exists a subgradient of $U,$ $\nabla U(\vx) \in \partial U(\vx),$ such that $\forall \vy \in \RR^d:$ 
    $$ U(\vy) \geq U(\vx) + \innp{\nabla U(\vx), \vy - \vx}.$$
    \item \label{assumption:holder} There exist $L < \infty$ and $\alpha \in [0, 1]$ such that $\forall \vx,\vy \in \mathbb{R}^d$, we have 
    \begin{equation}\label{eq:def-holder-cont-grad}
    \lv \nabla U(\vx) - \nabla U(\vy) \rv_2 \le L \lv \vx -\vy \rv_2^\alpha,
    \end{equation}
    where $\nabla U(\vx)$ denotes an arbitrary subgradient of $U$ at $\vx$. 
    \item  \label{assumption:fourthmomentbound} The distribution $p^*$ has a finite \emph{fourth moment}:
    \begin{align*}
        \int_{\vx \in \mathbb{R}^d} \lv \vx - \vx^* \rv_2^4 \cdot p^*(\vx) \dd \vx = \mathcal{M}_4 < \infty,
    \end{align*}
    where $\vx^* \in \argmin_{\vx \in \mathbb{R}^d} U(\vx)$ is an arbitrary minimizer of $U$.
\end{enumerate}
Assumption~(\textcolor{blue}{A}\ref{assumption:holder}) is known as the $(L, \alpha)$-weak smoothness or H\"{o}lder continuity of the (sub)gradients of $U.$ When $\alpha =1,$ it corresponds to the standard \emph{smoothness} (Lipschitz continuity of the gradients), while at the other extreme, when $\alpha = 0,$ $U$ is (possibly) \emph{non-smooth} and \emph{Lipschitz-continuous}. 

\paragraph{Properties of weakly smooth functions.} 
A property that follows directly from~\eqref{eq:def-holder-cont-grad} 
\iffullpaper
is that: 
\begin{equation}\label{eq:weak-smoothness}
    \quad U(\vy) \leq U(\vx) + \innp{\nabla U(\vx), \vy - \vx} + \frac{L}{1+\alpha}\|\vy - \vx\|^{1+\alpha}, \quad \forall \vx, \vy \in \RR^d.
\end{equation}
\fi
One of the most useful properties of weakly smooth functions that has been exploited in optimization is that they can be approximated by smooth functions to an arbitrary accuracy, at the cost of increasing their smoothness parameter~\cite{nesterov2015universal, devolder2014first}. This was shown in~\cite[Lemma 1]{nesterov2015universal} and is summarized in the following lemma for the special case of the unconstrained Euclidean setting.
\begin{lemma}\label{lemma:smooth-approx-for-weak-sm}
Let $U:\mathbb{\RR}^d \rightarrow \RR$ be a convex function that satisfies~\eqref{eq:def-holder-cont-grad} for some $L < \infty$ and $\alpha \in [0, 1].$ Then, for any $\delta > 0$ and 
$
M = \big(\frac{1}{\delta}\big)^{\frac{1-\alpha}{1+\alpha}} L^{\frac{2}{1+\alpha}},
$
we have that, $\forall \vx, \vy \in \RR^d:$
\iffullpaper
\begin{equation}\label{eq:inexact-grad-model}
\begin{aligned}
U(\vy) \leq U(\vx) + \innp{\nabla U(\vx), \vy - \vx}+ \frac{M}{2}\|\vy - \vx\|^2 + \frac{\delta}{2}.
\end{aligned}
\end{equation}
\fi
\end{lemma}

Furthermore, it is not hard to show that Eq.~\eqref{eq:inexact-grad-model} implies 
\citep[see][Section 2.2]{devolder2014first}:
\begin{align}
    \label{eq:smoothnesscondition}
    \lv \nabla U(\vx) - \nabla U(\vy) \rv_2 \le M \lv \vx - \vy\rv_2 + 2\sqrt{\delta M}
\end{align}
where $M = \left(\frac{1}{\delta}\right)^{\frac{1-\alpha}{1+\alpha}}\cdot L^{2/(1+\alpha)}$, 
as in Lemma~\ref{lemma:smooth-approx-for-weak-sm}. 

%
%
\paragraph{Gaussian smoothing.} 
Given $\mu \geq 0$, define the Gaussian smoothing $U_{\mu}$ of $U$ as:
$$
U_\mu (\vy) := \mathbb{E}_{\vxi}[U(\vy + \mu \vxi)],
$$
where $\vxi \sim \mathcal{N}(\mathbf{0},I_{d\times d}).$
The reason for considering the Gaussian smoothing $U_\mu$ instead of $U$ is that it generally enjoys better smoothness properties. In particular, $U_\mu$ is smooth even if $U$ is not. Here we review some basic properties of $U_\mu,$ most of which can be found in~\cite[Section 2]{nesterov2017random} for non-smooth Lipschitz functions. We generalize some of these results to weakly smooth functions. While the results can be obtained for arbitrary normed spaces, here we state all the results for the space $(\mathbb{R}^d, \, \|\cdot\|_2),$ which is the only setting considered in this paper.


The following lemma is a simple extension of the results from~\cite[Section 2]{nesterov2017random} and it establishes certain regularity conditions for Gaussian smoothing that will be used in our analysis.
\begin{restatable}{lemma}{lemsmoothingfunerr}\label{lemma:closeness-of-smoothing}
Let $U:\RR^d \rightarrow \RR$ be a convex function that satisfies Eq.~\eqref{eq:def-holder-cont-grad} for some $L < \infty$ and $\alpha \in [0, 1].$ Then:
\iffullpaper
\begin{itemize}
\item[(i)]
$
\forall \vx \in \RR^d: \quad |U_\mu(\vx) - U(\vx)| = U_\mu(\vx) - U(\vx) \leq \frac{L \mu^{1+\alpha} d^{\frac{1+\alpha}{2}}}{1+\alpha}.
$
\item[(ii)] 
$
\forall \vx, \vy \in \RR^d: \quad \|\nabla U_\mu(\vy) - \nabla U_\mu(\vx)\|_2 \leq \frac{L d^{\frac{1-\alpha}{2}}}{\mu^{1-\alpha} (1+\alpha)^{1-\alpha}} \|\vy - \vx\|_2.
$
\end{itemize}
\fi
\end{restatable}

Additionally, we show that Gaussian smoothing preserves strong convexity, stated in the following (simple) lemma. Recall that a differentiable function $\psi$ is $\lambda$-strongly convex if, $\forall \vx, \, \vy \in \RR^d:$
$$
\psi(\vy) \geq \psi(\vx) + \innp{\nabla \psi(\vx), \vy - \vx} + \frac{\lambda}{2}\|\vy - \vx\|_2^2.
$$
\begin{restatable}{lemma}{lemgstrsc}
Let $\psi: \RR^d \rightarrow \RR$ be $\lambda$-strongly convex. Then $\psi_{\mu}$ is also $\lambda$-strongly convex.  
\end{restatable}
%
%
\paragraph{Composite potentials and regularization.} 
To prove convergence of the continuous-time process (which requires strong convexity), we work with potentials that have the following composite form:
\begin{align}
    \label{def:Ubar}
    \bar{U}(\vx) := U(\vx) + \psi(\vx),
\end{align}
where $\psi(\cdot)$ is $m$-smooth and  $\lambda$-strongly convex. For obtaining guarantees in terms of convergence to $\bar{p}^* \propto e^{-\bar{U}},$ we do not need Assumption~(\textcolor{blue}{A}\ref{assumption:fourthmomentbound}), which bounds the fourth moment of the target distribution---this is only needed in establishing the results for $p^* \propto e^{-U}.$  

If the goal is to sample from a distribution $p^*(\vx) \propto e^{-U(\vx)}$ (instead of $\bar{p}^*(\vx) \propto e^{-\bar{U}(\vx)}$), then we need to ensure that the distributions $p^*$ and $\bar{p}^*$ are sufficiently close to each other. 
  This can be achieved by choosing $\psi(\vx) = \frac{\lambda}{2} \lv \vx- \vx' \rv_2^2,$ where $\lambda$ and $\|\vx'-\vx^*\|_2$ are sufficiently small, for an arbitrary $\vx^* \in \argmin_{\vx \in \RR^d} U(\vx)$ ({see Corollary~\ref{cor:reg-potentials-result} for precise details}).  

Note that by the triangle inequality, we have that:
\iffullpaper
\begin{equation}\label{eq:change-in-barU-grads}
\begin{aligned}
\lv \nabla \bar{U}(\vx) - \nabla \bar{U}(\vy) \rv_2 &\leq \|\nabla U(\vx) - \nabla U(\vy)\|_2 + \|\nabla \psi(\vx) - \nabla \psi(\vy)\|_2\\
&\leq L\|\vx - \vy\|_2^{\alpha} + m\|\vx - \vy\|_2.
\end{aligned}
\end{equation}
\fi
Thus, by~\eqref{eq:smoothnesscondition}, we have the following (deterministic) Lipschitz approximation of the gradients of $\bar{U}$: $\forall \vx,\vy\in \mathbb{R}^d$, any $\delta >0$, and $M = M(\delta)$ (as in Lemma~\ref{lemma:smooth-approx-for-weak-sm}):
\iffullpaper
\begin{align}\label{eq:Ubarsmoothness}
    \lv \nabla \bar{U}(\vx) - \nabla \bar{U}(\vy) \rv_2 \le M\lv \vx -\vy\rv_2 + m \lv \vx - \vy\rv_2 + 2\sqrt{\delta M}. 
\end{align}
\fi
On the other hand, for Gaussian-smoothed composite potentials, using Lemma~\ref{lemma:closeness-of-smoothing}, we have:
\iffullpaper
\begin{equation}\label{U-bar-mu-smoothness}
    \|\nabla \bar{U}_{\mu}(\vx) - \nabla \bar{U}_{\mu}(\vy)\|_2 \leq \bigg(\frac{L d^{\frac{1-\alpha}{2}}}{\mu^{1-\alpha} (1+\alpha)^{1-\alpha}} +  m\bigg)\|\vx - \vy\|_2.
\end{equation}
\fi
\paragraph{Distances between probability measures.} Given any two probability measures $P$ and $Q$ on $(\mathbb{R}^d,\mathcal{B}(\mathbb{R}^d))$, where $\mathcal{B}(\mathbb{R}^d)$ is the Borel $\sigma$-field of $\mathbb{R}^d$, the total variation distance between them is defined as
\begin{align*}
    \lv P - Q \rv_{\TV} := \sup_{A \in \mathcal{B}(\mathbb{R}^d)} \lvert P(A) - Q(A) \rvert.
\end{align*}
 The \emph{Kullback-Leibler} divergence between $P$ and $Q$ is defined as: 
\begin{align*}
\KL(P\lvert Q) := \mathbb{E}_{P}\left[\log\left( \frac{\dd P}{\dd Q }\right)\right],
\end{align*}
where $\dd P/\dd Q$ is the Radon-Nikodym derivative of $P$ with respect to $Q$. 

Define a \emph{transference plan} $\zeta$, a distribution on $(\mathbb{R}^d\times \mathbb{R}^d,\mathcal{B}(\mathbb{R}^d \times \mathbb{R}^d))$ such that $\zeta(A \times \mathbb{R}^d) = P(A)$ and $\zeta(\mathbb{R}^d \times A) = Q(A)$ for any $A \in \mathcal{B}(\mathbb{R}^d)$. Let $\Gamma(P,Q)$ denote the set of all such transference plans. Then the $2$-Wasserstein distance is defined as:
\iffullpaper
\begin{align*}
    W_2(P,Q) := \bigg(\inf_{\zeta \in \Gamma(P,Q)}\int_{\vx,\vy \in \mathbb{R}^d} \lv \vx - \vy\rv_2^2 \dd \zeta(\vx,\vy) \bigg)^{1/2}.
\end{align*}
\fi
\section{Sampling for composite potentials}\label{sec:main}

In this section, we consider the setting of composite potentials of the form $\bar{U}(\vx) = U(\vx) + \psi(\vx),$ where $U(\cdot)$ is $(L, \alpha)$-weakly smooth (possibly with $\alpha = 0,$ in which case $U$ is nonsmooth and Lipschitz) and $\psi(\cdot)$ is $m$-smooth and $\lambda$-strongly convex. We provide results for mixing times\footnote{\emph{Mixing time} is defined as the number of iterations needed to reach an $\varepsilon$ accuracy in either 2-Wasserstein or total variation distance.} of different variants of overdamped LMC in both 2-Wasserstein and total variation distance.

We first consider the deterministic smooth approximation of $U,$ which follows from Lemma~\ref{lemma:smooth-approx-for-weak-sm}. This approach does not require making any changes to the standard overdamped LMC. However, it leads to a polynomial dependence of the mixing time on $d$ and $1/\varepsilon$ only when $\alpha$ is bounded away from zero (namely, when $1/\alpha$ can be treated as a constant). 

We then consider another approach that relies on a Gaussian smoothing of $\bar{U}$ and that leads to a polynomial dependence of the mixing time on $d$ and $1/\varepsilon$ for all values of $\alpha.$ In particular, the approach leads to the mixing time for 2-Wasserstein distance that matches the best known mixing time of overdamped LMC when $U$ is smooth ($\alpha = 1$) -- $\widetilde{\mathcal{O}}(d/\varepsilon^2)$, and preserves polynomial-time dependence on $d$ and $1/\varepsilon$ even if $U$ is nonsmooth ($\alpha = 0$), in which case the mixing time scales as $\widetilde{\mathcal O}(d^{\frac{5}{2}}/\varepsilon^4).$ The analysis requires us to consider a minor modification to standard LMC in which we perturb by a Gaussian random variable the points at which $\nabla\bar{U}$ is queried. Note that it is unclear whether it is possible to obtain such bounds for~\eqref{eq:discrete-langevin} without this modification (see Appendix~\ref{app:shifted-lmc}). 

\subsection{First attempt: Deterministic approximation by a smooth function}\label{sec:main-deterministic}

In the optimization literature, deterministic smooth approximations of weakly smooth functions (as in Lemma~\ref{lemma:smooth-approx-for-weak-sm}) are generally useful for obtaining methods with optimal convergence rates~\citep{nesterov2015universal,devolder2014first}. A natural question is whether the same type of approximation is useful for bounding the mixing times of the Langevin Monte Carlo method invoked for potentials that are weakly smooth. 

We note that it is not obvious that such a deterministic approximation would be useful, 
as the deterministic error introduced by the smooth approximation causes an adversarial bias $2\sqrt{\delta M(\delta)}$ in the Lipschitz approximation of the gradients  (see Eq.~\eqref{eq:smoothnesscondition}). While this bias can be made arbitrarily small for values of $\alpha$ that are bounded away from zero, when $\alpha = 0,$ $M(\delta) = L^2/\delta,$ and the induced bias is constant for any value of $\delta.$

We show that it is possible to bound the mixing times of LMC when the potential is ``not too nonsmooth''. In particular, we show that the upper bound on the mixing time of LMC when applied to an $(L, \alpha)$-weakly smooth potential scales with {$\mathrm{poly}((\frac{1}{\varepsilon})^{1/\alpha})$} in both the 2-Wasserstein and total variation distance, which is polynomial in $1/\varepsilon$ for $\alpha$ bounded away from zero. Although we do not prove any lower bounds on the mixing time in this case, the obtained result aligns well with our observation that the deterministic bias cannot be controlled for the deterministic smooth approximation of a nonsmooth Lipschitz function, as explained above. 
%
%
Technical details are deferred to Appendix~\ref{app:det-approx}.



\subsection{Gaussian smoothing}\label{sec:main-gaussian}

The main idea is summarized as follows. Recall that LMC with respect to the potential $\bar{U}$ can be stated as:
\begin{equation}\tag{LMC}
    \vx_{k+1} = \vx_k - \eta \nabla \bar{U}(\vx_k) + \sqrt{2\eta}\vxi_{k},
\end{equation}
where $\vxi_k \sim \mathcal{N}(\mathbf{0}, I_{d\times d})$ are independent Gaussian random vectors. This method corresponds to the Euler-Mayurama discretization of the Langevin diffusion. 

Consider a modification of~\eqref{eq:discrete-langevin} in which we add another Gaussian term:
\begin{align} \label{eq:smoothed-langevin}
    \vx_{k+1} = \vx_k -\eta \nabla \bar{U}(\vx_k) + \sqrt{2\eta}\vxi_{k} + \mu \vomega_{k},
\end{align}
where $\vomega_{k} \sim \mathcal{N}(\mathbf{0}, I_{d\times d})$ and is independent of $\vxi_{k}$.
Observe that~\eqref{eq:smoothed-langevin} is simply another~\eqref{eq:discrete-langevin} with a slightly higher level of noise---$\sqrt{2\eta}\vxi_{k} + \mu \vomega_{k}$ instead of $\sqrt{2\eta}\vxi_{k}$. 
Let $\vy_k := \vx_k - \mu \vomega_{k-1}.$ Then:
\iffullpaper
\begin{equation}\tag{S-LMC}\label{eq:auxillary-sequence}
\begin{aligned}
\vy_{k+1} &= \vy_k + \mu \vomega_{k-1} - \eta \nabla \bar{U}(\vy_k + \mu \vomega_{k-1}) + \sqrt{2\eta}\vxi_{k} \\
 &= \vy_k  - \eta \left[\nabla \bar{U}(\vy_k + \mu \vomega_{k-1}) -\frac{\mu}{\eta}\vomega_{k-1} \right]+ \sqrt{2\eta}\vxi_{k}.  
\end{aligned}
\end{equation}
\fi
Taking expectations on both sides with respect to $\vomega_{k-1}$:
$$
\mathbb{E}_{\vomega_{k-1}}[\vy_{k+1}] = \vy_k - \eta \nabla \bar{U}_\mu(\vy_k) + \sqrt{2\eta}\vxi_{k},
$$
where $\bar{U}_\mu$ is the Gaussian smoothing of $\bar{U},$ as defined in Section~\ref{sec:prelims}. Thus, we can view the sequence $\{\vy_k\}$ in Eq.~\eqref{eq:auxillary-sequence} as obtained by simply transforming the standard LMC chain to another LMC chain using stochastic estimates $\nabla \bar{U}(\vy_k + \mu \vomega_{k-1}) -\frac{\mu}{\eta}\vomega_{k-1}$ of the gradients. 
%
However, the variance of this gradient estimate is too high to handle nonsmooth functions, and, as before, our bound on the mixing time of this chain blows up as $\alpha \downarrow 0$ (see Appendix~\ref{app:shifted-lmc}).

Thus, instead of working with the algorithm defined in~\eqref{eq:auxillary-sequence}, we correct for the extra induced variance and consider the sequence of iterates defined by:
\iffullpaper
\begin{align}
    \vy_{k+1}  = \vy_k  - \eta \nabla \bar{U}(\vy_k + \mu \vomega_{k-1}) + \sqrt{2\eta}\vxi_{k} . \tag{P-LMC}\label{eq:modifed-LMC}
\end{align}
\fi
This sequence will have a sufficiently small bound on the variance to obtain the desired results.

\begin{restatable}{lemma}{variancebound}\label{lemma:variancebound}
For any $\vx \in \mathbb{R}^d$, and $\vz\sim \mathcal{N}(\mathbf{0},I_{d\times d})$, let $G(\vx, \vz) := \nabla \bar{U}(\vx+\mu \vz)$ denote a stochastic gradient of $\bar{U}_{\mu}$. Then $G(\vx, \vz)$ is an unbiased estimator of $\nabla \bar{U}_\mu$ whose (normalized) variance satisfies:
\iffullpaper
\begin{align*}
  \sigma^2 : = \frac{\mathbb{E}_{\vz}\left[\left\lv\nabla \bar{U}_{\mu}(\vx) - G(\vx, \vz)\right\rv_2^2\right]}{d}
  \le 4d^{\alpha-1}\mu^{2\alpha}L^2 + 4 \mu^2 m^2.
\end{align*}
\fi
\end{restatable}
\begin{rem}
The variance from Lemma~\ref{lemma:variancebound} can be lowered by using multiple independent samples  to estimate $\nabla \bar{U}_\mu$ (instead of a single sample as in~\eqref{eq:modifed-LMC}). However, unlike in the case of nonsmooth optimization~\citep{duchi2012randomized}, such a strategy will \emph{not} reduce the mixing times reported here. This is because the variance from Lemma~\ref{lemma:variancebound} is already low enough to not be a limiting factor in the mixing time bounds.
\end{rem}

Let the distribution of the $k^{th}$ iterate $\vy_k$ be denoted by $\bar{p}_k$, and let $\bar{p}_{\mu}^*\propto \exp(-\bar{U}_{\mu})$ be the distribution with $\bar{U}_{\mu}$ as the potential. Our overall strategy for proving our main result is as follows. First, we show that the Gaussian smoothing does not change the target distribution significantly with respect to the Wasserstein distance, by bounding $W_2(\bar{p}^*, \bar{p}_\mu^*)$ (Lemma~\ref{lemma:wassersteincontrol}). Using Lemma~\ref{lemma:variancebound}, we then invoke a result on mixing times of Langevin diffusion with stochastic gradients, which allows us to bound $W_2(\bar{p}_k, \bar{p}^*_\mu)$. Finally, using the triangle inequality and choosing a suitable step size $\eta$, smoothing radius $\mu,$ and number of steps $K$ so that $W_2(\bar{p}^*, \bar{p}_\mu^*) + W_2(\bar{p}_K, \bar{p}^*_\mu) \leq \varepsilon,$ we establish our final bound on the mixing time of~\eqref{eq:modifed-LMC} 
in Theorem~\ref{thm:g-smoothing-mixing-time-mLMC}. 
\begin{restatable}{lemma}{wassersteinapproxerror} \label{lemma:wassersteincontrol} Let $\bar{p}^*$ and $\bar{p}_{\mu}^*$ be the distributions corresponding to the potentials $\bar{U}$ and $\bar{U}_{\mu}$ respectively. Then:
\iffullpaper
\begin{align*}
 W_2(\bar{p}^*,\bar{p}^*_{\mu}) \le  \frac{8}{\lambda}\Big( \frac{3}{2}+ \frac{d}{2}\log\Big(\frac{2(M+m)}{\lambda}\Big)\Big)^{1/2}\Big(\beta_{\mu} + \sqrt{{\beta_{\mu}}/{2} }\Big),
\end{align*}
\fi
where 
$
\beta_{\mu} := \beta_{\mu}(d, L, m, \alpha) = \frac{L\mu^{1+\alpha}d^{\frac{1+\alpha}{2}}}{\sqrt{2}(1+\alpha)} + \frac{m\mu^2 d}{2}.
$
\end{restatable}

Our main result is stated in the following theorem.
%
%
%
%
\begin{theorem}\label{thm:g-smoothing-mixing-time-mLMC}
Let the initial iterate $\vy_0$ be drawn from a probability distribution $\bar{p}_0$. If the step size $\eta$ satisfies $\eta < 2/(M+m+\lambda)$, then:
\iffullpaper
\begin{align*}
    W_2(\bar{p}_K,\bar{p}^*) & \le \left(1-\lambda \eta\right)^{K/2}W_2(\bar{p}_0,\bar{p}^*_{\mu}) + \left(\frac{2(M+m)}{\lambda}\eta d\right)^{1/2}+ \sigma\sqrt{\frac{ (1+\eta)\eta d}{\lambda}} \\
    &  \qquad  \qquad \qquad \qquad \qquad  \qquad \qquad \qquad + \frac{8}{\lambda}\bigg( \frac{3}{2}+ \frac{d}{2}\log\bigg(\frac{2(M+m)}{\lambda}\bigg)\bigg)^{1/2} \Big(\beta_{\mu}+ \sqrt{\beta_{\mu}/2}\Big), 
\end{align*}
\fi
\iffullpaper
where 
$\sigma^2 \le 4d^{\alpha-1}\mu^{2\alpha}L^2 + 4\mu^2 m^2$, $M =\frac{L d^{\frac{1-\alpha}{2}}}{\mu^{1-\alpha} (1+\alpha)^{1-\alpha}},$ and $\beta_{\mu} = \frac{L\mu^{1+\alpha}d^{\frac{1+\alpha}{2}}}{\sqrt{2}(1+\alpha)} + \frac{m\mu^2 d}{2}$. 
\fi

Further, if, for $\varepsilon \in (0, d^{1/4}),$ we choose $$K \ge \frac{1}{\lambda \eta} \log\left(\frac{3W_2(\bar{p}_0,\bar{p}^*_{\mu})}{\varepsilon }\right),$$ 
\iffullpaper
where
\begin{align*}
    \eta  \le  \frac{\varepsilon^2 \mu^{1-\alpha}\lambda}{1000(L+m) d^{\frac{3-\alpha}{2}}} \quad \text{ and }\quad \mu =  \frac{\varepsilon^{\frac{2}{1+\alpha}} \min\{\lambda^{\frac{2}{1+\alpha}}, 1\}}{300\sqrt{d}\big(\sqrt{m} + L^{\frac{1}{1+\alpha}}\big)\sqrt{10 + d\log\left(\varepsilon^{-2}{(m+L)d}/{\lambda}\right)}}, 
\end{align*}
\fi
%
then $W_2(\bar{p}_K,\bar{p}^*) \le \varepsilon$.
\end{theorem}
%
\begin{rem}
Treating $L, m, \lambda$ as constants and using the fact that $W_2(\bar{p}_0, \bar{p}^*_{\mu}) = {\cal O} (\mathrm{poly}(d/\varepsilon))$ \citep[see,][Lemma~13, by choosing the initial distribution $\bar{p}_0$ appropriately]{cheng2018underdamped},
we find that Theorem~\ref{thm:g-smoothing-mixing-time-mLMC} yields a bound of $K = \widetilde{\mathcal{O}}\left(d^{\frac{5-3\alpha}{2}}\big/\varepsilon^{\frac{4}{1+\alpha}}\right)$. When $\alpha = 1$ (the Lipschitz gradient case), we recover the known mixing time of $K = \widetilde{\mathcal{O}}(d/\varepsilon^2)$, while at the other extreme when $\alpha = 0$ (the nonsmooth Lipschitz potential case), we find that $K = \widetilde{\mathcal{O}}(d^{\frac{5}{2}}/\varepsilon^4)$.
\end{rem}

The choice of the smoothing radius $\mu$ is made such that it is large enough to ensure that the smoothed distribution $\bar{p}_{\mu}$ is sufficiently smooth, but not too large so as to ensure that the bias, $W_2(\bar{p}^*,\bar{p}_{\mu})$, is controlled.

\begin{proof}[Proof of Theorem \ref{thm:g-smoothing-mixing-time-mLMC}]
By the triangle inequality, 
\begin{align} \label{eq:wassersteintriangle-m-lmc}
    W_2(\bar{p}_K,\bar{p}^*) \le W_2(\bar{p}_K,\bar{p}^*_{\mu}) +  W_2(\bar{p}^*,\bar{p}^*_{\mu}).
\end{align}
To bound the first term, $W_2(\bar{p}_K,\bar{p}^*_{\mu})$, we invoke~\cite[Theorem 21]{durmus2019analysis} (see Theorem~\ref{thm:dalalyan-karagulyan} in Appendix~\ref{app:auxiliary}). Recall that $\bar{U}_{\mu}$ is continuously differentiable, $(M+m)$-smooth (with $M = \frac{L d^{\frac{1-\alpha}{2}}}{\mu^{1-\alpha} (1+\alpha)^{1-\alpha}} $), and $\lambda$-strongly convex. Additionally, the sequence of points  $\{\vy_k\}_{k=1}^{K}$ can be viewed as a sequence of iterates of overdamped LMC with respect to the potential specified by $\bar{U}_{\mu},$ where the iterates are updated using unbiased stochastic estimates of $\bar{U}_{\mu}$. Thus we have:
\iffullpaper
\begin{align}\label{eq:contractionofwasserstein-m-lmc}
    W_2(\bar{p}_K,\bar{p}^*_{\mu}) \le \left(1-\lambda \eta\right)^{K/2}W_2(\bar{p}_0,\bar{p}^*_{\mu}) + \left(\frac{2(M+m)}{\lambda}\eta d\right)^{1/2}+\sigma\sqrt{\frac{ (1+\eta)\eta d}{\lambda}},
\end{align}
\fi
 and by Lemma~\ref{lemma:variancebound}, $\sigma^2 \le 4d^{\alpha-1}\mu^{2\alpha}L^2 + 4 \mu^2 m^2.$

The last piece we need is control over the distance between 
$\bar{p}^*$ and $\bar{p}_{\mu}^*$. This is established above in Lemma~\ref{lemma:wassersteincontrol}
\iffullpaper, which gives:
\begin{align} \label{eq:wassersteinperturbnation-m-lmc}
     W_2(\bar{p}^*,\bar{p}^*_{\mu}) \le  \frac{8}{\lambda}\left( \frac{3}{2}+ \frac{d}{2}\log\left(\frac{2(M+m)}{\lambda}\right)\right)^{1/2} \left(\beta_{\mu}+ \sqrt{\beta_{\mu}/2}\right),
\end{align}
where $\beta_{\mu}$ is as defined above. Combining Eqs.~\eqref{eq:wassersteintriangle-m-lmc}-\eqref{eq:wassersteinperturbnation-m-lmc}, we get a bound on $W_2(\bar{p}_K,\bar{p}^*)$ in terms of the relevant problem parameters. This proves the first part of the theorem.
\fi

It is straightforward to verify that our choice of $\mu$ ensures that $W_2(\bar{p}^*,\bar{p}_{\mu}^*)\le \varepsilon/3$. The choice of $\eta$ ensures that $(2(M+m)\eta d/\lambda)^{1/2} \le \varepsilon/6$ and the choice of $K$ ensures that the initial error contracts exponentially to $\varepsilon/3$ (see the proof of Theorem~\ref{thm:TVcompositecontract} in Appendix~\ref{sec:omitted-pfs-main} for a similar calculation). This yields the second claim.
\end{proof}

Further, we show that this result can be generalized to total variation distance. 
\iffullpaper
\begin{restatable}{theorem}{tvmainresultcomposite}\label{thm:TVcompositecontract}
Let the initial iterate $\vy_0$ be drawn from a probability distribution $\bar{p}_0$. If we choose the step size such that $\eta < 2/(M+m+\lambda),$ then:
\begin{align*}
    \lv \bar{p}_K &- \bar{p}^* \rv_{\TV} \le \frac{L \mu^{1+\alpha}d^{(1+\alpha)/2}}{1+\alpha} +\frac{ \lambda\mu^2 d}{2}\\
   &+\sqrt{\frac{M+m}{4}\Big({\sqrt{{2d}/{\lambda}+2\lv\vx^*\rv_2^2}} + {\sqrt{{4d}/{\lambda}+4\lv\vx^*\rv_2^2+ 2W_2^2(\bar{p}_K,\bar{p}_{\mu^*})}} + {2\lv \vx^* \rv_2}\Big)W_2(\bar{p}_K,\bar{p}^*_{\mu})},
\end{align*}
where $ W_2(\bar{p}_K,\bar{p}^*_{\mu}) \le \left(1-\lambda \eta\right)^KW_2(\bar{p}_0,\bar{p}^*_{\mu}) + \frac{2(M+m)}{\lambda}(\eta d)^{1/2}+ \frac{\sigma^2 (\eta d)^{1/2}}{M+m+\lambda+\sigma\sqrt{\lambda}}$; $\sigma^2 \le 4d^{\alpha-1}\mu^{2\alpha}L^2 + 4\mu^2 m^2$, and $M =\frac{L d^{\frac{1-\alpha}{2}}}{\mu^{1-\alpha} (1+\alpha)^{1-\alpha}}$.

Further, if, for $\epsilon \in (0, 1],$ we choose 
$
\mu = \min\Big\{ \frac{\varepsilon^{\frac{1}{1+\alpha}}}{4 \max\{1, L^{\frac{1}{1+\alpha}}\}d^{1/2}},\; \sqrt{\frac{\varepsilon\lambda}{2 m^2 d}} \Big\}
   $ 
   and 
$$
 \bar{\varepsilon} = \frac{\varepsilon^2}{4 \max\{(M+m) (\sqrt{2d/\lambda + 2\|\vx^*\|_2^2} + 2\|\vx^*\|_2^2), 1\}},
 $$
then by choosing the step size $\eta$ and number of steps $K$ as
\begin{align*}
    \eta & \leq \frac{\bar{\varepsilon}^2 \lambda}{64 d (M+m)} \quad\text{ and }\quad K \geq \frac{\log(2W_2(\bar{p_0}, \bar{p}_{\mu}^*)/\bar{\varepsilon})}{\lambda \eta},
\end{align*}
we have $\lv \bar{p}_K - \bar{p}^* \rv_{\TV} \le \epsilon$.
\end{restatable}
\fi

\begin{rem}\label{remark:TV-composite}
Treating $L, \mu, \lambda, \|\vx^*\|$ as constants and using the fact that $W_2(\bar{p}_0,\bar{p}^*_{\mu}) = {\cal O}(\mathrm{poly}(\frac{d}{\varepsilon}))$ \citep[by][Lemma 13, along with an appropriate choice for the initial distribution]{cheng2018underdamped}, Theorem~\ref{thm:TVcompositecontract} gives a bound on the mixing time $K = \widetilde{\cal O}(d^{5 - 3\alpha}/\varepsilon^{\frac{7+\alpha}{1+\alpha}}).$  
When $\alpha = 1$ (Lipschitz gradients), we have $K = \widetilde{\mathcal{O}}(d^2/\epsilon^4)$, while when $\alpha = 0$ (nonsmooth Lipschitz potential) we have $K = \widetilde{\mathcal{O}}(d^5 /\epsilon^7)$. 
While the bound for the smooth case (Lipschitz gradients, $\alpha = 1$) is looser than the best known bound for LMC with a warm start~\citep{dalalyan2017theoretical}, 
we conjecture that it 
is improvable. The main loss is incurred when relating $W_2$ to KL distance, using an inequality from~\citet{polyanskiy2016wasserstein} (see Appendix~\ref{app:auxiliary}). If tighter inequalities were obtained, 
either relating $W_2$ and KL, or 
$W_2$ and TV, this result would immediately improve as a consequence. The results for LMC with non-Lipschitz gradients ($\alpha \in [0, 1)$) are novel. Finally, as a byproduct of our approach, we obtain the first bound for stochastic gradient LMC in TV distance (see Remark~\ref{remark:stochastic-TV} in Appendix~\ref{sec:omitted-pfs-main}). 
\end{rem}

\section{Sampling for regularized potentials}\label{sec:reg-potentials}
Consider now the case in which we are interested in sampling from a distribution $p^* \propto \exp(-U)$. As mentioned in Section~\ref{sec:prelims}, we can use the same analysis as in the previous section, by running \eqref{eq:modifed-LMC} with a regularized potential $\bar{U} = U + \lambda \lv \vx - \vx'\rv_2^2/2$, where $\vx' \in \mathbb{R}^d$. %
To obtain the desired result, the only missing piece is bounding the distance between $\bar{p}^* \propto \exp(-\bar{U})$ and $p^*$, leading to  
the following corollary of Theorem~\ref{thm:TVcompositecontract}.  

\begin{restatable}{corollary}{tvregpotentials}\label{cor:reg-potentials-result}
Let the initial iterate $\vy_0$ satisfy $\vy_0 \sim \bar{p}_0,$ for some distribution $\bar{p}_0$ and let $\bar{p}_K$ denote the distribution of $\vy_K$. If we choose the step-size $\eta$ such that $\eta < 2/(M+2\lambda)$, then:
\iffullpaper
\begin{align*}
    \lv \bar{p}_K -& p^* \rv_{\TV} \le \|\bar{p}_K - \bar{p}^*\|_{\TV} + \frac{\lambda \sqrt{{\cal M}_4}}{2} + \frac{\lambda\lv \vx' - \vx^*\rv_2^2}{2}, 
    %
\end{align*}
\fi
where $\|\bar{p}_K - \bar{p}^*\|_{\TV}$ is bounded as in Theorem~\ref{thm:TVcompositecontract} and ${\cal M}_4$ is the fourth moment of $p^*.$ 

Further, if, for $\varepsilon' \in (0, 1],$ we choose $\lambda = \frac{4 \epsilon'}{\sqrt{{\cal M}_4}+ \lv \vx' - \vx^*\rv_2^2}$ and all other parameters as in Theorem~\ref{thm:TVcompositecontract} for $\varepsilon = \varepsilon'/2,$ then, 
we have $\lv \bar{p}_K - p^* \rv_{\TV} \le \epsilon'$.
\end{restatable}
\begin{proof} 
{By the triangle inequality, 
$$
\|\bar{p}_K - p^*\|_{\TV} \leq \|\bar{p}_K - \bar{p}^*\|_{\TV} + \|p^* - \bar{p}^*\|_{\TV}.
$$
Applying Lemma~\ref{lemma:dalalyan-bnded-dist} from the appendix, 
\begin{align*}
    \|p^* - \bar{p}^*\|_{\TV} \leq &\; 
    \frac{1}{2}\bigg(\int_{\RR^d}(U(\vx) - \bar{U}(\vx))^2{p}^*(\vx)\dd\vx \bigg)^{1/2}\\
    =&\; \frac{1}{2}\bigg(\int_{\RR^d}\Big(\frac{\lambda}{2}\|\vx - \vx'\|_2^2\Big)^2{p}^*(\vx)\dd\vx \bigg)^{1/2} \\
     \le &\; \frac{1}{2}\bigg(2\int_{\RR^d}\Big(\frac{\lambda}{2}\|\vx - \vx^*\|_2^2\Big)^2{p}^*(\vx)\dd\vx\\ &+ 2\int_{\RR^d}\Big(\frac{\lambda}{2}\|\vx^* - \vx'\|_2^2\Big)^2{p}^*(\vx)\dd\vx \bigg)^{1/2}.
\end{align*}
Thus, using Assumption~(\textcolor{blue}{A}\ref{assumption:fourthmomentbound}), we get 
$$
    \|p^* - \bar{p}^*\|_{\TV} \leq \frac{\lambda}{2} \sqrt{{\cal M}_4} + \frac{\lambda \lv \vx' - \vx^*\rv_2^2}{2}. 
$$
}
The rest of the proof follows by 
Theorem~\ref{thm:TVcompositecontract}.
\end{proof}

\begin{rem}
Treating $L, \|\vx^*\|_2, \|\vx' - \vx^*\|_2$ as constants, the upper bound on the mixing time is $K = \widetilde{\mathcal{O}}(\frac{d^{5 - 3\alpha}{{\cal M}_4}^{3/2}}{\varepsilon^{\frac{10+4\alpha}{1+\alpha}}}).$ Thus, when $\alpha = 1,$ we have $K = \widetilde{\mathcal{O}}(\frac{d^2 {{\cal M}_4}^{3/2}}{\varepsilon^7}),$ while when $\alpha = 0,$  $K = \widetilde{\mathcal{O}}(\frac{d^5 {{\cal M}_4}^{3/2}}{\varepsilon^{10}}).$
\end{rem}
\section{Discussion}\label{sec:discussion}

We obtained polynomial-time theoretical guarantees for a variant of LMC---\eqref{eq:modifed-LMC}---that uses Gaussian smoothing and applies to target distributions with nonsmooth log-densities. The smoothing we apply is tantamount to perturbing the gradient query points in LMC by a Gaussian random variable, which is a minor modification to the standard method. 

Beyond its applicability to sampling from more general weakly smooth and nonsmooth target distributions, our work also has some interesting implications. For example, we believe our results can be extended to sampling from structured distributions with nonsmooth and nonconvex negative log-densities, following an argument from, e.g.,~\citet{cheng2018sharp}. It should also be possible to work with stochastic gradients instead of exact gradients by coupling our arguments  with the bounds in \citet{dalalyan2019user} or \cite{durmus2019analysis}. Further, it seems plausible that coupling our results with the results for derivative-free LMC~\citep[][which only applies to distributions with smooth and strongly convex log-densities]{shen2019non} would lead to a more broadly applicable derivative-free LMC.

Several other interesting directions for future research remain. For example, as discussed in Remark~\ref{remark:TV-composite} and Remark~\ref{remark:stochastic-TV} (Appendix~\ref{sec:omitted-pfs-main}), we conjecture that the asymptotic dependence on $d$ and $\varepsilon$ in our bounds on the mixing times for total variation distance (Theorem~\ref{thm:TVcompositecontract}) can be improved to match those obtained for the 2-Wasserstein distance (Theorem~\ref{thm:g-smoothing-mixing-time-mLMC}). Further, in standard settings of LMC with the exact gradients, Metropolis filter is often used to improve the convergence properties of LMC and it leads to lower mixing times \citep[see, e.g.,][]{dwivedi2018log}. However, the performance of Metropolis-adjusted LMC becomes unclear once the gradients are stochastic (as is the case for~\eqref{eq:modifed-LMC}). It is an interesting question whether a Metropolis adjustment can speed up~\eqref{eq:modifed-LMC}.

\subsection*{Acknowledgements}
We thank Fran\c{c}ois Lanusse for his useful pointers to the literature on applied Bayesian statistics with nonsmooth posteriors. 
This research was supported by the 
NSF grants CCF-1740855 and IIS-1619362, and the Mathematical Data Science program of the
Office of Naval Research under grant number N00014-18-1-2764. Part of this work was done while the authors were visiting  Simons Institute for the Theory of Computing.

{
\iffullpaper
\newpage
\fi
\balance
\bibliographystyle{plainnat}
\bibliography{references.bib}
}

\newpage
\flushleft{\textbf{\LARGE{Appendix}}}
\appendix
\section{Additional background}\label{app:auxiliary}

Here we state the results from related work that are invoked in our analysis.

First, the smooth approximations of the potentials used in this paper are pointwise larger than the original potentials, and have a bounded distance from the original potentials. This allows us to invoke the following lemma from~\cite{dalalyan2017theoretical}.

\begin{lemma}\label{lemma:dalalyan-bnded-dist}\cite[Lemma 3]{dalalyan2017theoretical}
Let $U$ and $\Tilde{U}$ be two functions such that $U(\vx) \leq \Tilde{U}(\vx),$ $\forall \vx \in \RR^d$ and both $e^{-U}$ and $e^{-\Tilde{U}}$ are integrable. Then the Kullback-Leibler divergence between the distributions defined by densities $p \propto e^{-U}$ and $\Tilde{p} \propto e^{-\Tilde{U}}$ can be bounded as:
$$
\KL(p |\Tilde{p}) \leq \frac{1}{2}\int_{\RR^d} (U(\vx) - \Tilde{U}(\vx))^2 p(\vx)\dd\vx.
$$
As a consequence, $\|p - \Tilde{p}\|_{\TV}\leq \frac{1}{2}\|U - \Tilde{U}\|_{L^2(p)}.$  
\end{lemma}

The next result that we will be invoking allows us to bound the Wasserstein distance between the target distributions corresponding to the composite potential $\bar{U}$ and its Gaussian smoothing $\bar{U}_{\mu}.$

\begin{lemma}\label{lemma:bolley-villani}\cite[Corollary 2.3]{bolley2005weighted}
Let $X$ be a measurable space equipped with a measurable distance $\rho$, let $p \geq 1,$ and let $\nu$ be a probability measure on $X$. Assume that there exist $\vx_0 \in X$ and $\gamma > 0$ such that $\int_X e^{\gamma \rho(\vx_0, \vx)^p}\dd\nu(\vx)$ is finite. Then, for any other probability measure $\mu$ on $X:$
$$
W_p(\mu, \nu) \leq C \Big[\KL(\mu | \nu)^{1/p} + \Big(\frac{\KL(\mu|\nu)}{2}\Big)^{1/(2p)}\Big],
$$
where
$$
C:= 2 \inf_{\vx_0 \in X, \gamma > 0} \bigg(\frac{1}{\gamma}\bigg(\frac{3}{2} + \log \int_{X}e^{\gamma \rho(\vx_0, \vx)^{p}}\dd\nu(\vx)\bigg)\bigg).
$$
\end{lemma}

Another useful result, due to~\cite{polyanskiy2016wasserstein}, lets us bound the KL-divergence between two distributions in terms of their $2$-Wasserstein distance. This is used to relate the TV distance between distributions to their respective Wasserstein distance in Section~\ref{sec:main-gaussian}. 
\begin{proposition}\cite[Proposition 1]{polyanskiy2016wasserstein}\label{prop:wassersteinstability} Let $Q(\vx)$ be a density on $\mathbb{R}^d$ such that for all $\vx \in \mathbb{R}^d$: $\lv \nabla \log Q(\vx) \rv_2 \le c_1 \lv \vx \rv_2 + c_2$ for some $c_1,c_2 \ge 0$. Then,
\begin{align*}
    \KL(P\lvert Q) \le \left( \frac{c_1}{2}\left[\sqrt{\mathbb{E}_{\vx \sim P}\left[\lv \vx \rv_2^2\right]}+ \sqrt{\mathbb{E}_{\vy \sim Q}\left[\lv \vy \rv_2^2\right]}\right]+c_2\right)W_2(P,Q).
\end{align*}
\end{proposition}
In particular, if $Q(\vx) \propto e^{-U(\vx)}$ for some $M$-smooth function $U,$ then we immediately have: 
$$
\|\nabla \log Q(\vx) - \nabla \log Q(\vx^*)\|_2 = \|\nabla \log Q(\vx)\|_2 \leq M\|\vx - \vx^*\|_2 \leq M\|\vx\|_2 + M \|\vx^*\|_2,
$$ 
where $\vx^* \in \argmin_{\vx \in \RR^d} U(\vx)$, and the assumption of the proposition is satisfied with 
$$ c_1 = M \quad\text{ and }\quad c_2 = M \|\vx^*\|_2.$$

We will be invoking a result from~\cite{dalalyan2019user} that bounds the Wasserstein distance between the target distribution $p^*$ and the distribution of the $K^{\mathrm{th}}$ iterate of LMC with stochastic gradients. The assumptions about the stochastic gradients $G(\vx, \vz)$ is that they are unbiased and their variance is bounded. Namely:
$$
\mathbb{E}_{\vz_k}[G(\vx_k, \vz_k) =  \nabla U(\vx_k),
$$
and
$$
\mathbb{E}_{\vz_k}[\|G(\vx_k, \vz_k) - \mathbb{E}_{\vz_k'}[G(\vx_k, \vz_k')] \|_2^2] \leq \sigma^2 d,
$$
where the diffusion term $\vxi_{k+1}$ is independent of $(\vz_1, ..., \vz_k)$. The random vectors $(\vz_1, ..., \vz_k)$ corresponding to the error of the gradient estimate are not assumed to be independent in~\cite{dalalyan2019user}; however, in our case it suffices to assume that they are, in fact, independent.

\begin{theorem}\cite[Theorem~21]{durmus2019analysis}\label{thm:dalalyan-karagulyan}
Let $p_K$ be the distribution of the $K^{\mathrm{th}}$ iterate of Langevin Monte Carlo with stochastic gradients, and let $p^* \propto e^{-U}.$ If $U$ is $M$-smooth and $\lambda$-strongly convex and the step size $\eta$ satisfies $\eta \leq \frac{2}{M + \lambda},$ then:
$$
W_2(p_K, p^*) \leq (1-\lambda\eta)^{K/2} W_2(p_0, p^*) + 
\left(\frac{2M\eta d}{\lambda}\right)^{1/2} + \sigma\sqrt{\frac{ (1+\eta)\eta d}{\lambda}}.
$$
\end{theorem}

Next we state the results from \cite{durmus2016high} that we use multiple times in our proofs to establish contraction of the solution of the Langevin continuous-time stochastic differential equation:
\begin{align}\label{eq:cont-langevin-diffusion}
    \dd \vy_t = -\nabla U(\vy_t) \dd t + \dd \mB_t,
\end{align}
where $\vy_0 \sim q_0$. Let the distribution of $\vy_t$ be denoted by $q_t$. 

\begin{theorem}\cite[Proposition 1]{durmus2016high}\label{thm:durmus-1} Let the function $U$ be $L$-smooth and $\lambda$-strongly convex, let $q_0 \sim \delta_{\vx}$ (the Dirac-delta distribution at $\vx$), and let $\vx^*$ be the minimizer of $U$. Then:
\begin{enumerate}
    \item For all $t\ge 0$ and $\vx \in \mathbb{R}^d$,
    \begin{align*}
        \int_{\mathbb{R}^d} \lv \vz - \vx^* \rv_2^2 \, q_t(\vz) \dd \vz \le \lv \vx - \vx^* \rv_2^2\, e^{-2\lambda t} + \frac{d}{\lambda}(1-e^{-2\lambda t}).
    \end{align*}
    \item The stationary distribution $p^* \propto \exp(-U)$ satisfies $\int_{\vy \in \mathbb{R}^d} \lv \vy - \vx^* \rv_2^2\, p^*(\vy) \dd \vy \le d/\lambda$. 
    \item For any $\vx \in \mathbb{R}^d$ and $t>0$, $W_2(q_t,p^*) \le e^{-\lambda t}\left\{ \lv \vx -\vx^*\rv_2 + (d/\lambda)^{1/2}\right\}.$ 
\end{enumerate}
\end{theorem}   

Finally, we provide a slight modification of~\cite[Lemma 5]{dalalyan2017theoretical} that we use in the proof of Theorem \ref{thm:tv-composite-weaklysmooth}.
\begin{lemma} \label{lemma:continuoustimecontract-tv-composite} Let the function $\bar{U} = U + \psi$,  where U is $(L,\alpha)$-weakly smooth and $\psi$ is $m$-smooth and $\lambda$-strongly convex. If the initial iterate is chosen as $\vy_0 \sim q_0 = \mathcal{N}(\vx^* ,(M+m)^{-1} I_{d\times d})$ and $\bar{p}^* \propto \exp(-\bar{U}),$ then:
\begin{align*}
    \lv q_t - \bar{p}^* \rv_{\TV} \le \exp\left\{\frac{d}{4}\log\left( \frac{M+m}{\lambda}\right) + \frac{\delta}{4} - \frac{t\lambda}{2}\right\},
\end{align*}
where $q_t$ is the distribution of $\vy_t$ that evolves according to \eqref{eq:cont-langevin-diffusion}.
\end{lemma}
\begin{proof}
Using the definition of $p^*$,
\begin{align*}
    \bar{p}^*(\vy)^{-1} =&\; e^{\bar{U}(\vy)} \int_{\mathbb{R}^d} e^{-\bar{U}(\vz)} \dd \vz =e^{\bar{U}(\vy)-\bar{U}(\vx^*)} \int_{\mathbb{R}^d} e^{-\bar{U}(\vz) + \bar{U}(\vx^*)} \dd \vz  \\
     \le & \exp\left\{\langle \nabla \bar{U}(\vx^*), \vy-\vx^* \rangle +\frac{M+m}{2}\lv \vy-\vx^* \rv_2^2 + \frac{\delta}{2}\right\}\cdot\int_{\mathbb{R}^d}\exp\left\{-\langle \nabla \bar{U}(\vx^*),\vz-\vx^*\rangle - \frac{\lambda}{2}\lv \vz - \vx^*\rv_2^2 \right\}\dd \vz\\
     \le & \left(\frac{2\pi}{\lambda}\right)^{d/2} \exp\left\{\frac{M+m}{2}\lv \vy - \vx^* \rv_2^2 + \frac{\delta}{2}\right\} .
\end{align*}
Thus, we have that the $\chi^2$-divergence between $q_0$ and $p^*$ is bounded by
\begin{align*}
    \chi^2 \left(q_0\lvert \bar{p}^*\right)
    &= \mathbb{E}_{\vy \sim \bar{p}^*} \left[\left(\frac{q_0(\vy)}{\bar{p}^*(\vy)}\right)^2\right]\\
    & = \left(\frac{2 \pi}{M+m}\right)^{-d} \int_{\mathbb{R}^d} \exp\left\{-(M+m)\lv \vy - \vx^* \rv_2^2\right\}p^*(\vy)^{-1} \dd \vy \\
    & \le \exp(\delta/2)\left(\frac{2 \pi}{M+m}\right)^{-d} \left(\frac{2\pi}{\lambda}\right)^{d/2} \int_{\mathbb{R}^d} \exp\left\{ -\frac{M+m}{2}\lv \vy - \vx^* \rv_2^2 \right\}\dd \vy \\
    & \le \exp(\delta/2)\left(\frac{2 \pi}{M+m}\right)^{-d} \left(\frac{2\pi}{\lambda}\right)^{d/2} \left(\frac{2\pi}{M+m}\right)^{d/2}\\
    & \le \exp(\delta/2) \left(\frac{M+m}{\lambda}\right)^{d/2}.
\end{align*}
By \cite[Lemma 1]{dalalyan2017theoretical} (which only relies on the strong convexity of $\bar{U}$), we know that:
\begin{align*}
    \lv q_t - \bar{p}^* \rv_{\TV} & \le \frac{\exp(-t\lambda/2)}{2} \chi^2(q_0 \lvert \bar{p}^*)^{1/2} , \qquad \forall t \ge 0.
\end{align*}
Combining this with the upper bound on the initial $\chi^2$ divergence completes the proof.
\end{proof}
\section{Proofs for Gaussian smoothing}

\lemsmoothingfunerr*
\begin{proof}$ $\newline
\noindent\textbf{Proof of Part (i).}
First, it is not hard to show that whenever $U$ is convex and $\mu > 0,$ $U_{\mu}(\vx) \geq U(\vx),$ $\forall \vx.$ 
By the definition of $U_\mu$ and using that $\vxi$ is centered, we have:
$$
U_\mu(\vx) - U(\vx) = \frac{1}{(2\pi)^{d/2}}\int_{\RR^d} \big[U(\vx + \mu \vxi) - U(\vx) - \mu \innp{\nabla U(\vx), \vxi}\big]e^{-\|\vxi\|_2^2/2}\dd \vxi.
$$
Applying Eq.~\eqref{eq:weak-smoothness}:
$$
|U_\mu(\vx) - U(\vx)| \leq \frac{L}{1+\alpha}\mu^{1+\alpha}\frac{1}{(2\pi)^{d/2}}\int_{\RR^d} \|\vxi\|_2^{1+\alpha}e^{-\|\vxi\|_2^2/2}\dd \vxi.
$$
Finally, using~\cite[Lemma 1]{nesterov2017random}, $\frac{1}{(2\pi)^{d/2}}\int_{\RR^d} \|\vxi\|^{1+\alpha}e^{-\|\vxi\|_2^2/2}\dd \vxi \leq d^{(1+\alpha)/2}.$

%
%
%
\noindent\textbf{Proof of Part (ii).} 
First, observe that, by Jensen's inequality and Eq.~\eqref{eq:def-holder-cont-grad}:
\begin{equation}\label{eq:holder-smoothness-of-U-mu}
    \begin{aligned}
        \|\nabla U_\mu(\vy) - \nabla U_\mu(\vx)\|_2 &\leq \frac{1}{(2\pi)^{d/2}}\int_{\RR^d} \|\nabla U(\vy + \mu \vxi) - \nabla U(\vx + \mu \vxi)\|_2 e^{-\|\vxi\|_2^2/2}\dd \vxi \\
        &\leq L \|\vy - \vx\|_2^{\alpha}.
    \end{aligned}
\end{equation}
Further, by~\cite[Eq.~(21)]{nesterov2017random}, the gradient of $U_\mu$ can be expressed as:
$$
\nabla U_\mu(\vx) = \frac{1}{\mu (2\pi)^{d/2}}\int_{\RR^d} U(\vx + \mu \vxi) \vxi e^{-\|\vxi\|_2^2/2}\dd \vxi. 
$$

Thus, applying Jensen's inequality, we also have:
\begin{equation}\label{eq:smoothness-with-grad-approx}
    \begin{aligned}
        \|\nabla U_\mu(\vy) - \nabla U_\mu(\vx)\|_2 &\leq \frac{1}{\mu (2\pi)^{d/2}}\int_{\RR^d} |U(\vx + \mu \vxi) - U(\vy+ \mu \vxi)|\cdot \|\vxi\|_2 e^{-\|\vxi\|_2^2/2}\dd \vxi.
    \end{aligned}
\end{equation}
Using Eq.~\eqref{eq:weak-smoothness}, we have that:
\begin{equation}\notag
    \begin{aligned}
        |U(\vx + \mu \vxi) - U(\vy+ \mu \vxi)| \leq & \min\Big\{\innp{\nabla U(\vy+\mu \vxi), \vx - \vy} + \frac{L}{1+\alpha}\|\vy - \vx\|_2^{1+\alpha},\\
        &\hspace{1cm}\innp{\nabla U(\vx + \mu \vxi), \vy - \vx} + \frac{L}{1+\alpha}\|\vy - \vx\|_2^{1+\alpha}\Big\}\\
        \leq & \frac{1}{2}\innp{\nabla U(\vy+\mu \vxi) - \nabla U(\vx + \mu \vxi), \vx - \vy} + \frac{L}{1+\alpha}\|\vy - \vx\|_2^{1+\alpha}\\
        \leq &\frac{L}{1+\alpha}\|\vy - \vx\|_2^{1+\alpha},
    \end{aligned}
\end{equation}
where the second inequality comes from the minimum being smaller than the mean, and the last inequality is by convexity of $U$ (which implies $\innp{\nabla U(\vx) - \nabla U(\vy), \vx - \vy} \geq 0,$ $\forall \vx, \, \vy$).
Thus, combining with Eq.~\eqref{eq:smoothness-with-grad-approx}, we have:
\begin{equation}\label{eq:smoothness-2}
\begin{aligned}
    \|\nabla U_\mu(\vy) - \nabla U_\mu(\vx)\|_2 &\leq \frac{L}{\mu(1+\alpha)}\|\vy - \vx\|_2^{1+\alpha} \frac{1}{(2\pi)^{d/2}}\int_{\RR^d} \|\vxi\|_2 e^{-\|\vxi\|_2^2/2}\dd \vxi\\
    &= \frac{L}{\mu(1+\alpha)}\|\vy - \vx\|_2^{1+\alpha} d^{1/2}.
\end{aligned}
\end{equation}
Finally, combining Eqs.~\eqref{eq:holder-smoothness-of-U-mu} and~\eqref{eq:smoothness-2}:
\begin{align*}
    \|\nabla U_\mu(\vy) - \nabla U_\mu(\vx)\|_2 &= \|\nabla U_\mu(\vy) - \nabla U_\mu(\vx)\|_2^{\alpha} \cdot \|\nabla U_\mu(\vy) - \nabla U_\mu(\vx)\|_2^{1-\alpha}\\
    &\leq L^{\alpha} \Big(\frac{L d^{1/2}}{\mu(1+\alpha)}\Big)^{1-\alpha}\|\vy - \vx\|_2\\
    &= \frac{L d^{\frac{1-\alpha}{2}}}{\mu^{1-\alpha} (1+\alpha)^{1-\alpha}} \|\vy - \vx\|_2,
\end{align*}
as claimed.
\end{proof}

\lemgstrsc*
\begin{proof}
By the definition of a Gaussian smoothing, $\forall \vx, \vy \in \RR^d$:
\begin{align*}
    \psi_\mu(\vy) - &\psi_{\mu}(\vx) - \innp{\nabla \psi_\mu(\vx), \vy - \vx}\\
    &= \frac{1}{(2\pi)^{d/2}}\int_{\RR^d}\Big( \psi(\vy + \mu \vxi) - \psi(\vx + \mu \vxi) - \innp{\nabla \psi(\vx + \mu \vxi), \vy - \vx}\Big)e^{-\|\vxi\|_2^2/2}\dd\vxi\\
    &\geq \frac{1}{(2\pi)^{d/2}} \int_{\RR^d}\frac{\lambda}{2}\|\vy - \vx\|^2 e^{-\|\vxi\|_2^2/2}\dd\vxi\\
    &= \frac{\lambda}{2}\|\vy - \vx\|^2,
\end{align*}
where we have used $\lambda$-strong convexity of $\psi.$
\end{proof}
\section{Mixing times for deterministic approximations of negative log-density}\label{app:det-approx}

In this section, we analyze the convergence of Langevin diffusion in the \emph{2-Wasserstein} distance and \emph{total variation} distance for target distributions of the form $\bar{p}^* \propto e^{-\bar{U}(\vx)},$ where $\bar{U}(\cdot) = U(\cdot) + \psi(\cdot)$, $U(\cdot)$ is $(L, \alpha)$-weakly-smooth, and $\psi(\cdot)$ is $m$-smooth and $\lambda$-strongly convex. 
The techniques we use here are an extension of similar techniques used previously by~\cite{dalalyan2017theoretical,durmus2016high}.

To analyze the convergence, in both cases we will use a coupling argument that  bounds the discretization error after Euler-Mayurama discretization is applied to the Langevin diffusion. Consider the first process which describes the exact continuous time process:
\begin{align}\label{def:continuousprocess}
    \dd\vx_t = - \nabla \bar{U}(\vx_t) \dd t + \sqrt{2}\dd\mB_t,
\end{align}
with initial condition $\vx_0 \sim p_0 \equiv q_0$. Let the distribution of $\vx_t$ be denoted by $q_t$. Let $p_0 \mathbb{P}_t$ denote the distribution of the entire path $\{\vx_s\}_{s=0}^t$. Consider a second process that describes the Euler-Mayurama discretization of \eqref{def:continuousprocess},
\begin{align}
    \label{def:discreteprocess}
    \dd{\tilde{\vx}}_t = -\vb_t(\tilde{\vx}_t) \dd t + \sqrt{2}\dd \mB_t,
\end{align}
with the same initial condition $\tilde{\bf{x}}_0  \sim p_0 \equiv q_0$,  $\vb_t(\tilde{\vx}_t) = \sum_{k=0}^{\infty}\nabla \bar{U}(\tilde{\vx}_{k\eta})\cdot \mathbb{I}\left[t\in [k\eta,(k+1)\eta)\right]$, and the same Brownian motion (synchronous coupling). Let the distribution of $\tilde{\vx}_t$ be denoted by $\tilde{p}_t.$ 

We will analyze the following (Langevin) iterative algorithm for which the initial point $\vx_0$ satisfies $\vx_0 \sim p_0 \equiv q_0$ and the $k^{\mathrm{th}}$ iterate is given by:
    \begin{align}\label{eq:discr-langevin}
        \vx_k = \vx_{k-1} -\eta \nabla \bar{U}(\vx_{k-1})+\sqrt{2\eta}\vxi_k,
    \end{align}
    where $\vxi_k \sim \mathcal{N}(\mathbf{0}, I_{d\times d}).$ Observe that this algorithm corresponds to the discretized process~\eqref{def:discreteprocess} with a fixed step size $\eta,$ and thus we will use $\bar{p}_k$ to denote the distribution of $\vx_k$.

\subsection{Guarantees for Wasserstein distance}

%
Define the difference process between $\vx_t$ and $\tilde{\vx}_t$ as the process $\vz_t$ which evolves according to:
\begin{align*}
    \dd\vz_t = - \left( \nabla \bar{U}(\vx_t) - \nabla \bar{U}(\vx_0)\right)\dd t.
\end{align*}
To bound the discretization error, here we want to bound the 2-Wasserstein distance between the distributions $q_{\eta},\tilde{p}_{\eta}$ after one step of size $\eta$. This is established in the following lemma.

\begin{lemma}[Discretization Error]\label{lemma:wasserstein-discr-err}
Let $\vx_0 \sim p_0$ for some probability distribution $p_0$ and $\tilde{\vx}_0 = \vx_0$ . Let $q_{\eta}$ denote the distribution of point $\vx_t$ defined by the process~\eqref{def:continuousprocess} and $\tilde{p}_t$ denote the distribution of point $\tilde{\vx}_t$ defined by the process~\eqref{def:discreteprocess}, as described above. If $\eta < 1/(2\lambda)$ then:
\begin{align*}
W_2^2(q_{\eta},\tilde{p}_{\eta}) &\leq \frac{8 d (M + m)^4 }{\lambda}\eta^4+8(M+m)^4\eta^4 W_2^{2}(p_0,\bar{p}^*)\\ & \qquad \qquad \qquad \qquad + 32 \delta (M + m)^2 M \eta^4 + 4d(M + m)^2 \eta^3 + 8 \delta M \eta^2. 
\end{align*}
\end{lemma}
\begin{proof}
By the definition of Wasserstein distance (as the infimum over all couplings) we have,
\begin{align*}
    W_2^2(q_{\eta},\tilde{p}_{\eta}) \le \mathbb{E}\left[ \lv \tilde{\vx}_{\eta} - \vx_{\eta}\rv_2^2\right]  = \mathbb{E}\left[ \lv \vz_{\eta}\rv_2^2\right] & = \mathbb{E}\left[ \left\lv -\int_{0}^{\eta} \left( \nabla \bar{U}(\vx_t) - \nabla \bar{U}(\vx_0)\right)\dd t\right\rv_2^2\right] \\
    & \le \eta \int_0^{\eta}\mathbb{E}\left[ \left\lv  \left( \nabla \bar{U}(\vx_t) - \nabla \bar{U}(\vx_0)\right)\right\rv_2^2\right] \dd t,
\end{align*}
where we have used Jensen's inequality. Continuing by using the smoothness of $\bar{U}$ (Eq. \eqref{eq:Ubarsmoothness}) and applying Young's inequality ($(a+b)^2 \le 2a^2 + 2b^2$), we get,
\begin{align*}
     W_2^2(q_{\eta},\tilde{p}_{\eta}) & \le \eta \int_{0}^{\eta} \mathbb{E}\left[2(M+m)^2 \lv \vx_t - \vx_0 \rv_2^2 + 8\delta M\right] \dd t\\
     & = 2(M+m)^2 \eta \int_{0}^{\eta} \mathbb{E}\left[ \lv \vx_t - \vx_0 \rv_2^2\right] \dd t + 8 \delta M \eta^2 \\
     & = 2(M+m)^2 \eta \int_{0}^{\eta} \mathbb{E}\left[ \left\lv \int_0^t \left(-\nabla \bar{U}(\vx_s)\dd s + \sqrt{2}\dd\mB_s\right) \right\rv_2^2\right] \dd t + 8 \delta M \eta^2,
\end{align*}
where the last equality is by the definition of the continuous process \eqref{def:continuousprocess}. 

By another application of Young's inequality:
\begin{align}
    \nonumber W_2^2 &(q_{\eta},\tilde{p}_{\eta}) - 8 \delta M \eta^2\\
    \nonumber &\leq  4 (M + m)^2 \eta \int_{0}^{\eta} \mathbb{E}\left[ \left\lv \int_0^t \nabla \bar{U}(\vx_s) \dd s  \right\rv_2^2\right] \dd t + 8 (M + m)^2 \eta \int_0^{\eta} \underbrace{\mathbb{E}\left[\lv \mB_t \rv_2^2\right]}_{= d \cdot t}\dd t \\
    \nonumber &\overset{(i)}{\le }  4 (M + m)^2\eta \int_{0}^{\eta} t \int_0^t \mathbb{E}\left[ \underbrace{\left\lv \nabla \bar{U}(\vx_s) - \nabla \bar{U}(\vx^*)  \right\rv_2^2}_{\le 2 (M + m)^2 \lv \vx_s - \vx^* \rv_2^2 + 16 \delta M}\right] \dd s \dd t + 4d(M + m)^2 \eta^3 \\
    & \le 16 (M + m)^4\eta \int_0^\eta t \int_0^t \mathbb{E}\left[\lv \vx_s - \vx^* \rv_2^2 \right] \dd s \dd t + 32 \delta (M + m)^2 M \eta^4 + 4d(M + m)^2 \eta^3, \label{eq:wassersteinbounddiscrete-onestep}
    \end{align}
 where $(i)$ is again by Jensen's inequality. We now look to bound the term $\mathbb{E}\left[\lv \vx_s - \vx^* \rv_2^2\right]$. By \cite[Proposition 1]{durmus2016high} (see the first part of Theorem \ref{thm:durmus-1}) we have
 \begin{align*}
     \mathbb{E}\left[\lv \vx_s - \vx^* \rv_2^2\right] \le \frac{d}{\lambda} + e^{-2\lambda s} \mathbb{E}\left[\lv \vx_0 - \vx^* \rv_2^2\right] = \frac{d}{\lambda} + e^{-2\lambda s} \mathbb{E}_{\vy \sim p_0} \left[\lv \vy - \vx^* \rv_2^2\right].
 \end{align*}
 Let $\vz \sim \bar{p}^*$ and assume it is optimally coupled to $\vy$,  so that $\mathbb{E}_{\vy \sim p_0,\vz \sim \bar{p}^*}\left[\lv \vy - \vz\rv_2^2\right] = W_2^2(p_0,\bar{p^*})$. Then, combining with the inequality above we get:
 \begin{align*}
     \mathbb{E}\left[\lv \vx_s - \vx^* \rv_2^2\right] &\le \frac{d}{\lambda} + e^{-2\lambda s} \mathbb{E}_{\vy \sim p_0} \left[\lv \vy - \vx^* \rv_2^2\right] \\
     & \le \frac{d}{\lambda} + e^{-2\lambda s} \mathbb{E}_{\vy \sim p_0,\vz \sim \bar{p}^*} \left[\lv \vy - \vz + \vz - \vx^* \rv_2^2\right]\\
     & \le \frac{d}{\lambda} + 2e^{-2\lambda s} \mathbb{E}_{\vy \sim p_0,\vz \sim \bar{p}^*} \left[\lv \vy - \vz  \rv_2^2\right] + 2  \mathbb{E}_{\vz \sim \bar{p}^*} \left[\lv \vz - \vx^* \rv_2^2\right] \\
     & \le \frac{3d}{\lambda} + 2e^{-2\lambda s}W_2^{2}(p_0,\bar{p}^*),
 \end{align*}
 where in the last inequality the bound on $\mathbb{E}_{\vz \sim \bar{p}^*}\left[\lv \vz - \vx^*\rv_2^2\right]\le d/\lambda$ follows from \cite[Proposition~1]{durmus2016high} (see second part of Theorem \ref{thm:durmus-1} in Appendix \ref{app:auxiliary}). Plugging this into \eqref{eq:wassersteinbounddiscrete-onestep} we get
    \begin{align*}
    W_2^2 (q_\eta, \tilde{p}_{\eta}) & \le \frac{48 d (M + m)^4 }{\lambda}\eta\int_0^\eta t \int_0^t \dd s \dd t + 16 (M + m)^4 \eta \cdot W_2^{2}(p_0,\bar{p}^*)\int_0^\eta t (1-e^{-2\lambda t}) \dd t
\\& \qquad \qquad +32 \delta (M + m)^2 M \eta^4 + 4d(M + m)^2 \eta^3+  8 \delta M \eta^2 \\
& \overset{(i)}{\le }\frac{48 d (M + m)^4 }{\lambda}\eta\int_0^\eta t \int_0^t \dd s \dd t + 16 (M + m)^4 \eta \cdot W_2^{2}(p_0,\bar{p}^*)\int_0^\eta t^2  \dd t
\\& \qquad \qquad +32 \delta (M + m)^2 M \eta^4 + 4d(M + m)^2 \eta^3+  8 \delta M \eta^2 \\
    & \le \frac{8 d (M + m)^4 }{\lambda}\eta^4+8(M+m)^4\eta^4 W_2^{2}(p_0,\bar{p}^*)\\ & \qquad \qquad \qquad + 32 \delta (M + m)^2 M \eta^4  +4d(M + m)^2 \eta^3+ 8 \delta M \eta^2.
\end{align*}
where $(i)$ follows as $\eta < 1/(2\lambda)$ which implies $1-e^{-2\lambda t} \le 2\lambda t$.
\end{proof}

\begin{theorem}
Let $\bar{p}_k$ denote the distribution of the $k^{\mathrm{th}}$ iterate $\vx_k$ of the discrete Langevin Monte Carlo given by Eq.~\eqref{eq:discr-langevin}, where $\vx_0 \sim \bar{p}_0.$ If the step size $\eta$ satisfies:
 $$ 
    \eta \le \min\bigg\{\frac{\min\{\lambda, \lambda^2\}\varepsilon^2}{90000 d}\cdot \frac{1}{\big((\frac{24}{\lambda \varepsilon})^{\frac{1-\alpha}{\alpha}}L^{\frac{1}{\alpha}}+m \big)},\, \frac{1}{2\lambda},\, \frac{\lambda}{36(M+m)}\bigg\},
    $$ 
  then, for any  
    \begin{align*}
        k \ge \frac{720000 d}{\min\{1, \lambda\} \varepsilon^2 \lambda^2}\left(\left(\frac{24}{\lambda \varepsilon}\right)^{\frac{1-\alpha}{\alpha}}L^{\frac{1}{\alpha}}+m \right) \log\left(\frac{W_2(\bar{p}_0,\bar{p}^*)}{\varepsilon} \right),
    \end{align*}
    we have that $W_2 (\bar{p}_k, \bar{p}^*) \leq \varepsilon,$ where $\bar{p}^* \propto e^{-\bar{U}}$ is the target distribution.
\end{theorem}

\begin{proof}
     By the triangle inequality, we have that at any step $k$:
    \begin{align*}
        &W_2(\bar{p}_k,\bar{p}^*) \le \underbrace{e^{-\lambda \eta}W_2(\bar{p}_{k-1},\bar{p}^*)}_{\text{Continuous Process Contraction}}\\
        &+ \underbrace{\frac{9 \sqrt{d} (M + m)^2 }{\sqrt{\lambda}}\eta^2+ 9(M+m)\eta^2 W_2(\bar{p}_{k-1},\bar{p}^*)+ 6 \sqrt{\delta} (M + m) \sqrt{M} \eta^2 + 2\sqrt{d}(M + m) \eta^{3/2} + 4 \sqrt{\delta M} \eta}_{\text{Discretization Error}},
    \end{align*}
    where the continuous process contraction follows from \cite[Proposition 1]{durmus2016high} (see third point in Theorem \ref{thm:durmus-1} in Appendix \ref{app:auxiliary}), while the discretization error is due to the Lemma \ref{lemma:wasserstein-discr-err}. First we club together the two terms that contain $W_2(\bar{p}_{k-1},\bar{p}^*)$ and observe that:
    \begin{align*}
        e^{-\lambda \eta} W_2(\bar{p}_{k-1},\bar{p}^*) + 9(M+m)\eta^2W_2(\bar{p}_{k-1},\bar{p}^*) &\le \left(1-\frac{\lambda\eta}{2} + 9(M+m)\eta^2\right)W_2(\bar{p}_{k-1},\bar{p}^*) \\
        & \overset{(i)}{\le} \left(1-\frac{\lambda\eta}{2} + \frac{\lambda \eta}{4} \right)W_2(\bar{p}_{k-1},\bar{p}^*) \\
        & =   \left(1-\frac{\lambda\eta}{4} \right)W_2(\bar{p}_{k-1},\bar{p}^*)\\
        &\le e^{-\lambda\eta/8}W_2(\bar{p}_{k-1},\bar{p}^*),
    \end{align*}
    where $(i)$ follows as $\eta < \lambda/(36(M+m))$.
    
    Assume that:
    $$
    2\sqrt{d}(M+m)\eta^{3/2} \ge \max\left\{6 \sqrt{\delta} (M + m) \sqrt{M} \eta^2, \frac{9 \sqrt{d} (M + m)^2 }{\sqrt{\lambda}}\eta^2\right\}.
    $$ 
    (It is not hard to check that this assumption holds for the choice of the step size $\eta$ and for $\delta$ specified below.) 
    Unrolling the recursive inequality for $W_2(\bar{p}_k, \bar{p}^*)$ over $k$ steps, we get:
    \begin{align*}
        W_2(\bar{p}_k,\bar{p}^*) &\le e^{-\lambda k \eta/ 8} W_2(\bar{p}_{0},\bar{p}^*) + \left(4\sqrt{\delta M}+ 6\sqrt{d}(M+m)\eta^{3/2}\right)\eta\sum_{s=0}^{\infty}e^{-\lambda s\eta/8} \\&\le e^{-\lambda k \eta/8} W_2(\bar{p}_{0},\bar{p}^*) + \frac{4\sqrt{\delta M}\eta}{1-e^{-\lambda \eta/8}} + \frac{6\sqrt{d}(M+m)\eta^{3/2}}{1-e^{-\lambda \eta/8}}.
    \end{align*}
    Recalling that $M = (\frac{1}{\delta})^{\frac{1-\alpha}{1+\alpha}}L^{\frac{2}{1+\alpha}},$ we further have:    
    \begin{align*}
        W_2(\bar{p}_k,\bar{p}^*) 
        & \le e^{-\lambda k \eta/8} W_2(\bar{p}_{0},\bar{p}^*)+\frac{4\sqrt{\delta \cdot(\frac{1}{\delta})^{\frac{1-\alpha}{1+\alpha}}L^{\frac{2}{1+\alpha}}}\eta}{1-e^{-\lambda \eta/8}}+\frac{6\sqrt{d}( (\frac{1}{\delta})^{\frac{1-\alpha}{1+\alpha}}L^{\frac{2}{1+\alpha}}+m)\eta^{3/2}}{1-e^{-\lambda \eta/8}}\\
        & \overset{(i)}{\le} e^{-\lambda k \eta/8} W_2(\bar{p}_{0},\bar{p}^*) + \frac{64 L^{\frac{1}{1+\alpha}}\delta^{\frac{\alpha}{1+\alpha}}\eta}{\lambda \eta}+\frac{96\sqrt{d}( (\frac{1}{\delta})^{\frac{1-\alpha}{1+\alpha}}L^{\frac{2}{1+\alpha}}+m)\eta^{3/2}}{\lambda \eta}\\
        & = e^{-\lambda k \eta/8} W_2(\bar{p}_{0},\bar{p}^*) + \frac{64 L^{\frac{1}{1+\alpha}}\delta^{\frac{\alpha}{1+\alpha}}}{\lambda}+\frac{96\sqrt{d}( (\frac{1}{\delta})^{\frac{1-\alpha}{1+\alpha}}L^{\frac{2}{1+\alpha}}+m)\eta^{1/2}}{\lambda },
    \end{align*}
    where (i) uses $1 - e^{-a} \geq \frac{a}{2},$ which holds for any $a \in [0, 1].$
    
    Choosing $\delta = \Big(\frac{\lambda \varepsilon}{24 L^{\frac{1}{1+\alpha}}}\Big)^{\frac{1+\alpha}{\alpha}}$ and recalling that the step-size is:
    $$ 
    \eta \le \frac{\lambda^2 \varepsilon^2 }{3^2 \cdot (96)^2 d ((\frac{1}{\delta})^{\frac{1-\alpha}{1+\alpha}}L^{\frac{2}{1+\alpha}}+m)^2}= \frac{\lambda^2\varepsilon^2}{90000 d}\cdot \frac{1}{\left(\left(\frac{24}{\lambda \varepsilon}\right)^{\frac{1-\alpha}{\alpha}}L^{\frac{1}{\alpha}}+m \right)},
    $$ 
    we have:
    \begin{align*}
        W_2(\bar{p}_k,\bar{p}^*) &\le e^{-\lambda k \eta} W_2(\bar{p}_{0},\bar{p}^*) + \frac{2\varepsilon}{3}.
    \end{align*}
    Thus, for any:
    \begin{align*}
        k \ge \frac{720000 d}{\min\{1, \lambda\} \varepsilon^2 \lambda^2}\left(\left(\frac{24}{\lambda \varepsilon}\right)^{\frac{1-\alpha}{\alpha}}L^{\frac{1}{\alpha}}+m \right) \log\left(\frac{W_2(\bar{p}_0,\bar{p}^*)}{\varepsilon} \right),
    \end{align*}
    we have $W_2(\bar{p}_k,\bar{p}^*)\le \varepsilon$, as claimed.
\end{proof}
\subsection{Guarantees for total variation distance}

Let $p_0\tilde{\mathbb{P}}_t$ denote the distribution of the entire stochastic process $\{\tilde{\vx}_s\}_{s\in [0, t]}$ described by~\eqref{def:discreteprocess}.  
Similar to~\cite{dalalyan2017theoretical}, we use Girsanov's formula~\cite[Chapter 8]{oksendal2003stochastic} to control the Kullback-Leibler divergence between the distributions $p_0\mathbb{P}_t$ and $p_0\tilde{\mathbb{P}}_t$.
\begin{align}\label{def:girsanovformula}
    \KL\left(p_0 \mathbb{P}_t \lvert p_0 \tilde{\mathbb{P}}_t\right)  = \frac{1}{4} \int_{0}^t\mathbb{E}\left[\lv \nabla \bar{U}(\tilde{\vx}_s) + \vb(\tilde{\vx}_s) \rv_2^2\right]\dd s.
\end{align}
The application of this identity allows us to bound the discretization error, as in the following lemma.
\begin{lemma}[Discretization Error Bound] \label{lem:discretebound}Let $\mathbf{x}^*$ be a point such that $\nabla \bar{U}(\mathbf{x}^*) = 0$. 
Then, for any integer $K\ge 1$ we have:
\begin{align*}
     \KL\left(p_0\mathbb{P}_{K\eta} \lvert p_0 \tilde{\mathbb{P}}_{K\eta}\right) & \le \frac{(M+m)^3\eta^2}{9} \mathbb{E}_{\vy\sim p_0}\left[\lv \vy - \vx^* \rv_2^2\right]+\frac{(M+m)^2Kd\eta^3}{9}\\ &\qquad \qquad \qquad +\frac{(M+m)^2(K+1)\delta\eta^2}{9}+\frac{(M+m)^2Kd\eta^2}{2} + \delta \eta K M.
\end{align*}
\end{lemma}
\begin{proof} By Girsanov's formula Eq. \eqref{def:girsanovformula} and the definition of $\vb(\tilde{\vx})$ we have,
\begin{align*}
     \KL\left(p_0\mathbb{P}_{K\eta} \lvert p_0 \tilde{\mathbb{P}}_{K\eta}\right) & = \frac{1}{4}\int_{0}^{K\eta} \mathbb{E}\left[\lv \nabla \bar{U}(\tilde{\vx}_s) + \vb(\tilde{\vx}_s) \rv_2^2\right]\dd s \\
    & = \frac{1}{4} \sum_{k=0}^{K-1} \int_{k\eta}^{(k+1)\eta} \mathbb{E}\left[\lv \nabla \bar{U}(\tilde{\vx}_s) - \nabla \bar{U}(\tilde{\vx}_{k\eta})\rv_2^2\right] \dd s
\end{align*}
Using the smoothness property of $\bar{U}$, Eq.~\eqref{eq:Ubarsmoothness}, and Young's inequality, we get:
\begin{align}
    \nonumber  \KL\left(p_0\mathbb{P}_{K\eta} \lvert p_0 \tilde{\mathbb{P}}_{K\eta}\right) & \le \frac{(M+m)^2}{2} \sum_{k=0}^{K-1} \int_{k\eta}^{(k+1)\eta} \mathbb{E}\left[\lv  \tilde{\vx}_s -  \tilde{\vx}_{k\eta}\rv_2^2\right] \dd s + \sum_{k=0}^{K-1} \int_{k\eta}^{(k+1)\eta} \delta M \dd s \\
    & = \frac{(M+m)^2}{2} \sum_{k=0}^{K-1} \int_{k\eta}^{(k+1)\eta} \mathbb{E}\left[\lv  \tilde{\vx}_s -  \tilde{\vx}_{k\eta}\rv_2^2\right] \dd s + \delta \eta K M.  \label{eq:discretizationboundfirstpart}
\end{align}
Let us unpack and bound the first term on the right hand side. By the definition of $\tilde{\vx}_s$, we have that for each $k \in \{0,\ldots, K-1\}$:
\begin{align*}
      \int_{k\eta}^{(k+1)\eta} \mathbb{E}\left[\lv  \tilde{\vx}_s -  \tilde{\vx}_{k\eta}\rv_2^2\right] \dd s&  =  \int_{k\eta}^{(k+1)\eta} \mathbb{E}\left[\left\lv  \int_{k\eta}^{s}\left(\nabla \bar{U}(\tilde{\vx}_{k\eta})+\sqrt{2}\dd \mB_r\right)\dd r\right\rv_2^2\right] \dd s \\
      & = \int_{k\eta}^{(k+1)\eta} \left(\mathbb{E}\left[\left\lv\nabla \bar{U}(\tilde{\vx}_{k\eta})\right\rv_2^2(s-k\eta)^2\right]+2d(s-k\eta)\right)\dd s.
\end{align*}
Plugging this back into Eq.~\eqref{eq:discretizationboundfirstpart}, we get:
\begin{align*}
     \KL\left(p_0\mathbb{P}_{K\eta} \lvert p_0 \tilde{\mathbb{P}}_{K\eta}\right) &\le \frac{(M+m)^2\eta^3}{6}\sum_{k=0}^{K-1}\mathbb{E}\left[\left\lv\nabla \bar{U}(\tilde{\vx}_{k\eta})\right\rv_2^2\right] + \frac{dK(M+m)^2\eta^2}{2} + \delta \eta K M.
\end{align*}
By invoking Lemma \ref{lem:gradientbound}, we get the desired result.
\end{proof}

\begin{lemma}\label{lem:gradientbound} Let $\eta \le 1/(2(M+m))$ and let $K \ge 1$ be an integer. Then:
\begin{align*}
    \eta \sum_{k=0}^{K-1} \mathbb{E}\left[\lv \nabla \bar{U}(\tilde{\vx}_{k\eta})\rv_2^2\right] \le \frac{2(M+m)}{3}\mathbb{E}\left[\lv \tilde{\vx}_0 - \vx^* \rv_2^2\right] + \frac{4(M+m)K\eta d}{3} + \frac{2 (K+1)\delta}{3}.
\end{align*}
\end{lemma}
\begin{proof} Let $\bar{U}^{(k)} := \bar{U}(\tilde{\vx}_{k\eta})$. Then, by the smoothness of $\bar{U}$ (Eq.~\eqref{eq:inexact-grad-model}), we have:
\begin{align*}
    \bar{U}^{(k+1)} &\le \bar{U}^{(k)} + \langle \nabla \bar{U}(\tilde{\vx}_{k\eta}),\tilde{\vx}_{(k+1)\eta} - \tilde{\vx}_{k\eta}\rangle + \frac{M+m}{2}\left\lv \tilde{\vx}_{(k+1)\eta} - \tilde{\vx}_{k\eta}\right\rv_2^2 + \frac{\delta}{2} \\
    & = \bar{U}^{(k)} - \eta \lv \nabla \bar{U}(\tilde{\vx}_{k\eta})\rv_2^2 + \sqrt{2\eta}\langle  \nabla \bar{U}(\tilde{\vx}_{k\eta}),\vxi_k\rangle + \frac{M+m}{2}\left\lv \eta  \nabla \bar{U}(\tilde{\vx}_{k\eta}) -\sqrt{2\eta}\vxi_k \right\rv_2^2 +\frac{\delta}{2},
\end{align*}
where $\vxi_k = \int_{s=k\eta}^{(k+1)\eta} \dd \mB_s$ is independent Gaussian noise. Taking expectations on both sides:
\begin{align*}
    \mathbb{E}\left[\bar{U}^{(k+1)}\right] & \le \mathbb{E}\left[\bar{U}^{(k)}\right] -\eta \mathbb{E}\left[\lv \nabla \bar{U}(\tilde{\vx}_{k\eta})\rv_2^2\right] + \frac{M+m}{2}\eta^2\mathbb{E}\left[\lv \nabla \bar{U}(\tilde{\vx}_{k\eta})\rv_2^2\right] + (M+m)\eta d + \frac{\delta}{2}\\
    & = \mathbb{E}\left[\bar{U}^{(k)}\right]  -\eta\left(1-\frac{(M+m)\eta}{2}\right)\mathbb{E}\left[\lv \nabla \bar{U}(\tilde{\vx}_{k\eta})\rv_2^2\right]+ (M+m)\eta d + \frac{\delta}{2}.
\end{align*}
Rearranging the above inequality and summing from $k=0$ to $K-1$, we get that:
\begin{align*}
    \eta \sum_{k=0}^{K-1} \mathbb{E}\left[\lv \nabla \bar{U}(\tilde{\vx}_{k\eta})\rv_2^2\right] &\le \frac{4}{3}\mathbb{E}\left[\bar{U}^{(0)} - \bar{U}^{(K)}\right] + \frac{4(M+m)K\eta d}{3} + \frac{2 K\delta}{3}.
\end{align*}
Let $\bar{U}^* = \inf_{\vx \in \mathbb{R}^d}\bar{U}(\vx)$. Therefore, we have $\bar{U}^{K} \ge \bar{U}^*$ and $\bar{U}^{(0)} - \bar{U}^* \le (M+m)\lv \tilde{\vx}_0 - \vx^* \rv_2^2/2 + \delta/2$. Combining this with the inequality above, we finally have:
\begin{align*}
    \eta \sum_{k=0}^{K-1} \mathbb{E}\left[\lv \nabla \bar{U}(\tilde{\vx}_{k\eta})\rv_2^2\right] &\le \frac{2(M+m)}{3}\mathbb{E}\left[\lv \tilde{\vx}_0 - \vx^* \rv_2^2\right] + \frac{4(M+m)K\eta d}{3} + \frac{2 (K+1)\delta}{3},
\end{align*}
as claimed.
\end{proof}

\begin{theorem} \label{thm:tv-composite-weaklysmooth}Let the initial point be drawn from a Normal distribution $\tilde{\vx}_0 \sim p_0 \equiv \mathcal{N} (\vx^*, (M+m)^{-1}I_{d\times d})$, where $\vx^*$ is a point such that $\nabla \bar{U}(\vx^*) = 0$. Let the step size satisfy $\eta < 1/(2(M+m))$. Then, for any integer $K\ge 1,$ we have:
\begin{align*}
    \lv \bar{p}_{K} -\bar{p}^*\rv_{\TV} \le & \frac{1}{2}\exp\left(\frac{d}{4}\log\left(\frac{M+m}{\lambda}\right)+\frac{\delta}{4}-\frac{K\eta \lambda}{2}\right)+ (M+m)\eta\sqrt{\frac{d}{18}} +(M+m)\sqrt{\frac{Kd\eta^3}{18}}\\ &+(M+m)\eta\sqrt{\frac{(K+1)\delta}{18}}+(M+m)\eta\sqrt{\frac{Kd}{2}} + \sqrt{\frac{\delta \eta K M}{2}}.
\end{align*}
\end{theorem}
\begin{proof}
By applying the triangle inequality to total variation distance, we have:
\begin{align*}
    \lv \tilde{p}_{K\eta} - \bar{p}^* \rv_{\TV} &\le \lv q_{K\eta} - \bar{p}^* \rv_{\TV} + \lv \tilde{p}_{K\eta} - q_{K\eta} \rv_{\TV}.
\end{align*}
Recall that, by definition, $\bar{p}_K$ (distribution of the $K^{th}$ iterate) is the same as $\tilde{p}_{K\eta}$, and $q_{K\eta}$ denotes the distribution of the solution to continuous process defined by \eqref{def:continuousprocess} at time $K\eta$. We start off by choosing the initial distribution to be a Gaussian $\tilde{\vx}_0 \sim {\cal N}(\vx^*,(M+m)^{-1}I_{d\times d})$. Therefore, by Lemma~\ref{lemma:continuoustimecontract-tv-composite}, we have that $ \lv q_{K\eta} - \bar{p}^* \rv_{\TV}$ can be bounded as:
\begin{align*}
    \lv q_{K\eta} - \bar{p}^* \rv_{\TV} &\le \frac{1}{2}\exp\left(\frac{d}{4}\log\left(\frac{M+m}{\lambda}\right)+\frac{\delta}{2}-\frac{K\eta \lambda}{2}\right).
\end{align*}
While by Lemma \ref{lem:discretebound}, which holds for a fixed initial point $\vx$, combined with the convexity of KL-divergence, we get that,
\begin{align*}
    \lv \tilde{p}_{K\eta} - q_{K\eta} \rv_{\TV} \le &
    \lv p_0\mathbb{P}_{K\eta} - p_0 \tilde{\mathbb{P}}_{K\eta} \rv_{\TV} \le  \left( \frac{1}{2}\KL\left(p_0\mathbb{P}_{K\eta} \lvert p_0 \tilde{\mathbb{P}}_{K\eta}\right)\right)^{1/2}\\
     \le &\sqrt{\frac{(M+m)^{3}\eta^2}{18} \mathbb{E}_{\vy\sim p_{0}}\left[\lv \vy - \vx^* \rv_2^2\right]}+\sqrt{\frac{(M+m)^2Kd\eta^3}{18}}\\ & 
    +\sqrt{\frac{(M+m)^2(K+1)\delta\eta^2}{18}}+\sqrt{\frac{(M+m)^2Kd\eta^2}{2}} + \sqrt{\frac{\delta \eta K M}{2}},
\end{align*}
where the first inequality is by the data-processing inequality and the second is by Pinsker's inequality. It is a simple calculation to show that $\mathbb{E}_{\vy\sim p_{0}}\left[\lv \vy - \vx^* \rv_2^2\right] = d/(M+m)$.
Combining this with the inequality above yields the desired claim.
\end{proof}

\begin{corollary}\label{cor:det-TV-bound}
In the setting of the theorem above, if we choose
\begin{align*}
    K & \ge \max\left\{\beta,\frac{d}{4\eta \lambda}\log\left( \frac{M+m}{\lambda}\right)+\frac{\delta}{4\eta\lambda}\right\}, \qquad \delta = \min\Bigg\{\bigg[\frac{\lambda\epsilon^2}{8d\log(\frac{M+m}{\lambda} ) L^{\frac{2}{1+\alpha}}}\bigg]^{\frac{1+\alpha}{2\alpha}},1\Bigg\},\\  \text{ and }\eta & \le \min\left\{1,\frac{1}{2\beta(M+m)},\frac{\lambda \epsilon^2}{32d^2(M+m)^2\log\left( \frac{M+m}{\lambda}\right) }\right\} 
\end{align*}
for some $\beta \ge 1$, where $M = M(\delta) = \left(\frac{1}{\delta}\right)^{\frac{1-\alpha}{1+\alpha}}L^{\frac{2}{1+\alpha}},$ then, $\lv \bar{p}_{K} -\bar{p}^*\rv_{\TV} \le \min\left\{\epsilon,1\right\}$.
\end{corollary}
\begin{proof} The proof follows by invoking the theorem above and by elementary algebra.
\end{proof}

\begin{rem} In the corollary above, if we treat $L,\beta,$ and $\lambda$ as constants, then we find that the mixing time $K$ scales as $\widetilde{\mathcal{O}}(d^{\frac{1+2\alpha}{\alpha}}/\epsilon^{\frac{2}{\alpha}})$. This recovers the rate obtained in \cite{dalalyan2017theoretical} when no warm start is used of $K= \widetilde{\mathcal{O}}(d^3/\epsilon^2)$ in the smooth case, $\alpha = 1$. However, as the potential $U$ gets nonsmooth, that is, $\alpha \to 0$, the mixing time blows up.
\end{rem}

In this section we have established results in the setting when we sample from  distributions with composite potential functions $\bar{U} = U+\psi$, where $U$ is $(L,\alpha)$-weakly smooth and $\psi$ is $m$-smooth and $\lambda$-strongly convex. If however, we are interested in sampling from a distribution with potential $U$ that is $(L,\alpha)$-weakly smooth, then we can add a small regularization to the potential exactly as we do in Section \ref{sec:reg-potentials} to obtain results similar to Corollary \ref{cor:reg-potentials-result}. Again these bounds on the mixing time would blow up as $\alpha \to 0$, but would be polynomial in $d$ and $\epsilon$ when $\alpha$ is sufficiently far from $0$.

\section{Shifted Langevin Monte Carlo}\label{app:shifted-lmc}

Here, we focus on bounding the mixing time of the sequence defined in Eq.~\eqref{eq:auxillary-sequence}, to which we refer as the Shifted Langevin Monte Carlo. Recall that this sequence is given by:
\begin{align}
\vy_{k+1} &= \vy_k + \mu \vomega_{k-1} - \eta \nabla \bar{U}(\vy_k + \mu \vomega_{k-1}) + \sqrt{2\eta}\vxi_{k} \nonumber\\
& = \vy_k  - \eta \left[\nabla \bar{U}(\vy_k + \mu \vomega_{k-1}) -\frac{\mu}{\eta}\vomega_{k-1} \right]+ \sqrt{2\eta}\vxi_{k}.  \notag
\end{align}

The only difference compared to the Perturbed Langevin method analyzed in Section~\ref{sec:main-gaussian} is in the bound on the variance, established in the following lemma.
\begin{lemma} \label{lemma:variancebound-bad}
For any $\vx \in \mathbb{R}^d$, and $\vz\sim \mathcal{N}(\mathbf{0},I_{d\times d})$, let $G(\vx, \vz) := \nabla \bar{U}(\vx+\mu \vz) - \frac{\mu}{\eta}\vz$ denote a stochastic gradient of $\bar{U}_{\mu}$. Then $G(\vx, \vz)$ is an unbiased estimator of $\nabla \bar{U}_\mu$ whose (normalized) variance can be bounded as:
\begin{align*}
  \sigma^2 : = \frac{\mathbb{E}_{\vz}\left[\left\lv\nabla \bar{U}_{\mu}(\vx) - G(\vx, \vz)\right\rv_2^2\right]}{d} \le 8d^{\alpha-1}\mu^{2\alpha}L^2 + 8 \mu^2 m^2 + \frac{2\mu^2}{\eta^2}.
\end{align*}
\end{lemma}
\begin{proof} 
Recall that by definition of $\bar{U}_{\mu}$, we have $\nabla \bar{U}_{\mu}(\vx) = \mathbb{E}_{\vw}\left[\bar{U}(\vx+ \mu \vw)\right]$, where $\vw \sim \mathcal{N}(\mathbf{0},I_{d\times d})$, and is independent of $\vz$. Clearly, $\mathbb{E}_{\vz}[G(\vx, \vz)] = \nabla \bar{U}_{\mu}(\vx).$

We now proceed to bound the variance of $G(\vx, \vz).$ First, using Young's inequality (which implies $(a+b)^2 \leq 2(a^2 + b^2),$ $\forall a, b$) and that $\vz\sim \mathcal{N}(\mathbf{0},I_{d\times d})$, we have:
\begin{align*}
    \mathbb{E}_{\vz}\left[\left\lv\nabla \bar{U}_{\mu}(\vx) - G(\vx, \vz)\right\rv_2^2\right] &\leq 2\mathbb{E}_{\vz}\left[\lv\mathbb{E}_{\vw}\left[\bar{U}(\vx+ \mu \vw)\right] - \nabla \bar{U}(\vx+ \mu\vz)\rv_2^2\right] + \frac{2\mu^2\mathbb{E}_{\vz}\left[\lv \vz\rv_2^2\right]}{\eta^2}\\
    &= 2\mathbb{E}_{\vz}\left[\lv\mathbb{E}_{\vw}\left[\bar{U}(\vx+ \mu \vw)\right] - \nabla \bar{U}(\vx+ \mu\vz)\rv_2^2\right] + \frac{2\mu^2 d}{\eta^2}.
\end{align*}
The rest of the proof follows the same argument as the proof of Lemma~\ref{lemma:variancebound} and is omitted.
\end{proof}


%
%
We can now establish the following theorem.
\begin{theorem}\label{thm:g-smoothing-mixing-time-bad}
Let the initial iterate satisfy $\vy_0 \sim \bar{p}_0$. If we choose the step-size such that we have $\eta < 2/(M+m+\lambda),$ then:
\begin{align*}
    W_2(\bar{p}_K,\bar{p}^*) & \le \left(1-\lambda \eta\right)^{K/2}W_2(\bar{p}_0,\bar{p}^*_{\mu}) + \left(\frac{2(M+m)}{\lambda}\eta d\right)^{1/2}+ \sigma\sqrt{\frac{ \eta d}{\lambda}} \\& \qquad \qquad  \qquad \qquad + \frac{8}{\lambda}\left( \frac{3}{2}+ \frac{d}{2}\log\left(\frac{2(M+m)}{\lambda}\right)\right)^{1/2} \left(\beta_{\mu}+ \sqrt{\beta_{\mu}/2}\right), 
\end{align*}
where $\sigma^2 \le 8d^{\alpha-1}\mu^{2\alpha}L^2 + 8 \mu^2 m^2 + \frac{2\mu^2}{\eta^2}$, $M =\frac{L d^{\frac{1-\alpha}{2}}}{\mu^{1-\alpha} (1+\alpha)^{1-\alpha}}$ and $\beta_{\mu} = \frac{L\mu^{1+\alpha}d^{\frac{1+\alpha}{2}}}{\sqrt{2}(1+\alpha)} + \frac{m\mu^2 d}{2}$.
\end{theorem}

\begin{proof}
By a triangle inequality, we can bound above the Wasserstein distance between $p_K$ and $\bar{p}^*$ by:
\begin{align} \label{eq:wassersteintriangle}
    W_2(\bar{p}_K,\bar{p}^*) \le W_2(\bar{p}_K,\bar{p}^*_{\mu}) +  W_2(\bar{p}^*,\bar{p}^*_{\mu}).
\end{align}
To bound the first term---$W_2(\bar{p}_K,\bar{p}^*_{\mu})$---we invoke \cite[Theorem~21]{durmus2019analysis} (see Theorem~\ref{thm:dalalyan-karagulyan} in Appendix~\ref{app:auxiliary}). Recall that $\bar{U}_{\mu}$ is continuously differentiable, $(M+m)$-smooth (with $M = \frac{L d^{\frac{1-\alpha}{2}}}{\mu^{1-\alpha} (1+\alpha)^{1-\alpha}} $) and $\lambda$-strongly convex. Additionally, $\{\vy_k\}_{k=1}^{K}$ can be viewed as iterates of overdamped Langevin MCMC with respect to the potential specified by $\bar{U}_{\mu}$ and is updated using unbiased noisy gradients of $\bar{U}_{\mu}$. Thus, we get:
\begin{align}\label{eq:contractionofwasserstein}
    W_2(\bar{p}_K,\bar{p}^*_{\mu}) \le \left(1-\lambda \eta\right)^{K/2}W_2(\bar{p}_0,\bar{p}^*_{\mu}) + \left(\frac{2(M+m)}{\lambda}\eta d\right)^{1/2}+ \sigma\sqrt{\frac{ \eta d}{\lambda}}.
\end{align}
 As was shown in Lemma~\ref{lemma:variancebound-bad},
\begin{align} \label{eq:sigma-definition}
    \sigma^2 \le 8d^{\alpha-1}\mu^{2\alpha}L^2 + 8 \mu^2 m^2 + \frac{2\mu^2}{\eta^2}.
\end{align}
The last piece we need is control over the distance between the distributions $\bar{p}^*$ and $\bar{p}_{\mu}^*$. Notice that by Lemma \ref{lemma:closeness-of-smoothing} it is possible to control the point-wise distance between $\bar{U}$ and $\bar{U}_{\mu}$, and hence the likelihood ratio and the KL-divergence between $\bar{p}$ and $\bar{p}_{\mu}$. We can then use Lemma~\ref{lemma:bolley-villani} to upper bound the Wasserstein distance between these distribution by the KL-divergence. These calculations are worked out in detail in Lemma \ref{lemma:wassersteincontrol} to get:
\begin{align} \label{eq:wassersteinperturbnation}
     W_2(\bar{p}^*,\bar{p}^*_{\mu}) \le  \frac{8}{\lambda}\left( \frac{3}{2}+ \frac{d}{2}\log\left(\frac{2(M+m)}{\lambda}\right)\right)^{1/2} \left(\beta_{\mu}+ \sqrt{\beta_{\mu}/2}\right),
\end{align}
where $\beta_{\mu}$ is as defined above. By combining Eqs. \eqref{eq:wassersteintriangle}-\eqref{eq:wassersteinboundperturb} we get a bound on $W_2(p_K,\bar{p}^*)$ in terms of the relevant problem parameters. 
\end{proof}


Consider the following choice of $\mu,\eta$ and $K$:
\begin{align}\label{eq:choiceofeta-smoothing}
    K &\ge \frac{1}{\lambda \eta} \log\left(\frac{10W_2(p_0,\bar{p}^*_{\mu})}{\varepsilon }\right),
    \qquad \mu  = \left[\frac{\eta Ld^{\frac{1-\alpha}{2}}}{2\sqrt{\lambda}}\right]^{\frac{1}{2-\alpha}} \text{ and, }\\
    \eta & = \min\left\{\left(\frac{\varepsilon}{1000}\right)^{\frac{2(2-\alpha)}{\alpha}}\frac{\lambda^{\frac{4(2-\alpha)^2}{\alpha(3-\alpha)}}}{L^{2/\alpha}d^{\frac{3-2\alpha}{\alpha}}},\frac{\lambda^{\frac{3-\alpha}{2(1+\alpha)}}\varepsilon^{\frac{2(1-\alpha)}{1+\alpha}}}{5000C_3^{\frac{2(1-\alpha)}{1+\alpha}}L^{\frac{1-\alpha}{1+\alpha}}d^{\frac{(3+\alpha)(1-\alpha)}{2(1+\alpha)}}}\right\}. \nonumber
\end{align}
and consider a regime of \emph{target accuracy} $C_1<\varepsilon < C_2$, for two positive constants $C_1,\, C_2$, such that the following holds:
\begin{enumerate}
    \item $M>m$.
    \item $\frac{\mu^2}{\eta^2} > 4\max\{d^{\alpha-1}\mu^{2\alpha}L^2,\mu^2m^2\}$.
    \item $d\log\left(\frac{2(M+m)}{\lambda}\right)>3$.
    \item $  \frac{L\mu^{1+\alpha}d^{\frac{1+\alpha}{2}}}{\sqrt{2}(1+\alpha)}> \frac{m\mu^2d}{2}$ and $\beta_{\mu} <1$.
    \item $\log\left(\frac{2(M+m)}{\lambda}\right) \le C_3$.
\end{enumerate}
Observe that $K$ blows up as $\alpha \downarrow 0,$ since $\eta$ scales with $(\frac{1000}{\epsilon})^{2(2-\alpha)/\alpha},$ which tends to zero as $\alpha \downarrow 0,$ for any $\varepsilon < 1000.$

A constant $C_2(d,\alpha,L,m,\lambda)$ will exist as our parameters $\mu$ and $\eta$ are monotonically increasing functions of $\eps$ and hence $M \propto 1/\mu^{1-\alpha}$ is a monotonically decreasing function of $\varepsilon$. We choose a lower bound on $C_1\le\varepsilon$ to simplify the presentation of our corollary that follows to ensure that Condition 5 specified above holds; it is possible to get rid of this condition and it would only change the results by poly-logarithmic factors.

\begin{corollary} Under the conditions of Theorem~\ref{thm:g-smoothing-mixing-time-bad}, there exist positive constants $C_1,\, C_2,$ and $C_3$ such that if $C_1 < \varepsilon < C_2$, then under the choice of $\eta,\,\mu,$ and $K$ in Eq.~\eqref{eq:choiceofeta-smoothing}, we have:
\begin{align*}
    W_2(\bar{p}_K,\bar{p}^*) \le \varepsilon.
\end{align*}
\end{corollary}
\begin{proof}
    The proof follows by invoking Theorem \ref{thm:g-smoothing-mixing-time-bad} and using elementary algebra.
\end{proof} 
%
%
\section{Omitted results and proofs from Section~\ref{sec:main}}\label{sec:omitted-pfs-main}
\variancebound*
\begin{proof} 
Recall that by definition of $\bar{U}_{\mu}$, we have $\nabla \bar{U}_{\mu}(\vx) = \mathbb{E}_{\vw}\left[\bar{U}(\vx+ \mu \vw)\right]$, where $\vw \sim \mathcal{N}(\mathbf{0},I_{d\times d})$, and is independent of $\vz$. Clearly, $\mathbb{E}_{\vz}[G(\vx, \vz)] = \nabla \bar{U}_{\mu}(\vx).$

We now proceed to bound the variance of $G(\vx, \vz).$ First, by the definition of $G(\vx, \vz):$
\begin{align*}
    \mathbb{E}_{\vz}\left[\left\lv\nabla \bar{U}_{\mu}(\vx) - G(\vx, \vz)\right\rv_2^2\right] & = \mathbb{E}_{\vz}\left[\lv\mathbb{E}_{\vw}\left[\nabla\bar{U}(\vx+ \mu \vw)\right] - \nabla \bar{U}(\vx+ \mu\vz)\rv_2^2\right].
\end{align*}
By Jensen's inequality, 
$$
\mathbb{E}_{\vz}\left[\lv\mathbb{E}_{\vw}\left[\nabla\bar{U}(\vx+ \mu \vw)\right] - \nabla \bar{U}(\vx+ \mu\vz)\rv_2^2\right] 
\leq \mathbb{E}_{\vz,\vw}\left[\lv \nabla \bar{U}(\vx+ \mu\vw)-\nabla \bar{U}(\vx+ \mu\vz)\rv_2^2\right].
$$
Hence, using~\eqref{eq:change-in-barU-grads} and applying Young's inequality ($(a+b)^2 \leq 2(a^2 + b^2),$ $\forall a, b$), we further have:
\begin{align*}
    \mathbb{E}_{\vz}\left[\left\lv\nabla \bar{U}_{\mu}(\vx) - G(\vx, \vz)\right\rv_2^2\right] 
    &\leq \mathbb{E}_{\vz,\vw}\left[\Big ( L\lv  \mu(\vw-\vz)\rv_2^{\alpha} + m \|\mu(\vw-\vz)\|_2\Big)^2 \right] \\
    &\leq 2 L^2 \mu^{2\alpha}\mathbb{E}_{\vz,\vw}\left[ \lv\vw-\vz\rv_2^{2\alpha} \right] + 2 m^2\mu^2 \mathbb{E}_{\vz,\vw}[\|\vw-\vz\|_2^2].
\end{align*}
Observe that $f(y) = y^{\alpha}$ is a concave function, $\forall \alpha \in [0, 1].$ Hence, we have that $\mathbb{E}_{\vz,\vw}\left[ \lv\vw-\vz\rv_2^{2\alpha} \right] \leq \big(\mathbb{E}_{\vz,\vw}\left[ \lv\vw-\vz\rv_2^{2} \right]\big)^{\alpha}.$ As $\vw$ and $\vz$ are independent, $\vw - \vz \sim \mathcal{N}(\mathbf{0}, 2 I_{d\times d}).$ Thus, we finally have:
\begin{align*}
    \mathbb{E}_{\vz}\left[\left\lv\nabla \bar{U}_{\mu}(\vx) - G(\vx, \vz)\right\rv_2^2\right] 
    &\leq 4d^{\alpha}\mu^{2\alpha}L^{2} + 4 d \mu^2 m^2, 
\end{align*}
as claimed.
\end{proof}

\wassersteinapproxerror*
\begin{proof} By \cite[Corollary 2.3]{bolley2005weighted} (see Lemma~\ref{lemma:bolley-villani} in Appendix~\ref{app:auxiliary}), we have:
\begin{align} \label{eq:wassersteinboundperturb}
    W_2(\bar{p}^*,\bar{p}^*_{\mu}) \le C_{\bar{U}_{\mu}} \cdot \left(\sqrt{\KL(\bar{p}^* \lvert \bar{p}^*_{\mu})} + \left(\frac{\KL(\bar{p}^*\lvert \bar{p}^*_{\mu})}{2}\right)^{1/4}\right),
\end{align}
where 
\begin{align*}
    C_{\bar{U}_{\mu}} : = 2\inf_{\vy \in \mathbb{R}^d,\gamma > 0} \left(\frac{1}{\gamma}\left( \frac{3}{2} + \log\left( \int_{\RR_d} e^{\gamma \lv \vy-\vx\rv_2^2} \bar{p}^*_{\mu}(\vx)\dd\vx\right)\right) \right)^{1/2}.
\end{align*}
First, let us control the constant $C_{\bar{U}_{\mu}}$. Without loss of generality, let $\bf{0}$ be the global minimizer of $\bar{U}_{\mu}$ (the minimizer is unique, as the function is strongly convex) and let $\bar{U}_{\mu}(\mathbf{0}) = 0$ (as the target distribution is invariant to constant shift in the potential). Choose $\vy = \bf{0}$ and $\gamma = \lambda/4$. By smoothness and strong convexity of $\bar{U}_{\mu}$ (see Section~\ref{sec:prelims}): 
$$
\lambda \lv \vx \rv_2^2/2\le \bar{U}_{\mu}(\vx) \le (M+m)\lv \vx \rv_2^2/2.$$ 
Therefore:
\begin{align}
    \nonumber C_{\bar{U}_{\mu}} 
    &\le \frac{8}{\lambda} \Big(\frac{3}{2} +  \log\Big( \int_{\RR_d} e^{\lambda \lv \vx\rv_2^2/4}\bar{p}^*_{\mu}(\vx)\dd\vx\Big)\Big)^{1/2} 
    \le \frac{8}{\lambda} \bigg(\frac{3}{2} +  \log\bigg( \frac{\int_{\RR_d} e^{\lambda \lv \vx\rv_2^2/4 }e^{-\lambda\lv \vx \rv_2^2/2}\dd \vx}{\int_{\RR_d} e^{-(M+m)\lv \vx\rv_2^2/2}\dd\vx}\bigg) \bigg)^{1/2}\\
    &\le \frac{8}{\lambda} \Big(\frac{3}{2} +  \log\Big( \frac{(4\pi/\lambda)^{d/2}}{(2\pi/(M+m))^{d/2}}\Big) \Big)^{1/2} \le \frac{8}{\lambda}\Big( \frac{3}{2}+ \frac{d}{2}\log\Big(\frac{2(M+m)}{\lambda}\Big)\Big)^{1/2}.  \label{eq:boundonC}
\end{align}
Next, we can control the Kullback-Leibler divergence between the distributions by using \cite[Lemma 3]{dalalyan2017theoretical} (see Lemma~\ref{lemma:dalalyan-bnded-dist} in Appendix~\ref{app:auxiliary}). Using Lemma~\ref{lemma:closeness-of-smoothing}, we have  $0 \leq \bar{U}_{\mu} - \bar{U} \leq \frac{L \mu^{1+\alpha}d^{(1+\alpha)/2}}{1+\alpha} + \frac{m\mu^2 d}{2}$. Therefore:
\begin{align}
    \KL(\bar{p}^*\lvert \bar{p}^*_{\mu}) &\le \frac{1}{2}\int \left(\bar{U}(\vx)-\bar{U}_{\mu}(\vx)\right)^2 \bar{p}^*(\vx)\dd\vx\notag\\
    &\leq \Big(\frac{L \mu^{1+\alpha}d^{(1+\alpha)/2}}{1+\alpha} + \frac{m\mu^2 d}{2}\Big)^2 = {\beta_\mu}^2.\label{eq:KLboundperturb}
\end{align}
Combining Eqs.~\eqref{eq:wassersteinboundperturb}-\eqref{eq:KLboundperturb}, we get the claimed result.
\end{proof}

\tvmainresultcomposite*
\begin{proof}By a triangle inequality, we can upper bound the total variation distance between $\bar{p}_K$ and $\bar{p}^*$:
\begin{align}\label{eq:TVtriangle-composite}
    \lv \bar{p}_K - \bar{p}^* \rv_{\TV} \le \lv \bar{p}_K - \bar{p}^*_{\mu} \rv_{\TV} + \lv \bar{p}^* - \bar{p}^*_{\mu} \rv_{\TV}.
\end{align}
Same as in the proof of Theorem~\ref{thm:g-smoothing-mixing-time-mLMC}, to bound the Wasserstein distance between $\bar{p}_K$ and $\bar{p}_{\mu}^*,$ we  invoke~\cite[Theorem~21]{durmus2019analysis} (see Theorem \ref{thm:dalalyan-karagulyan} in Appendix \ref{app:auxiliary}), which leads to:
\begin{align}\label{eq:contractionofwasserstein-m-lmc-good-composite}
    W_2(\bar{p}_K,\bar{p}^*_{\mu}) \le \left(1-\lambda \eta\right)^{K/2}W_2(\bar{p}_0,\bar{p}^*_{\mu}) + \left(\frac{2(M+m)}{\lambda}\eta d\right)^{1/2}+ \sigma\sqrt{\frac{ (1+\eta)\eta d}{\lambda}},
\end{align}

Our next step is to relate this bound on the $W_2(\bar{p}_K,\bar{p}^*_{\mu})$ to the total variation distance between these distributions. Let $\bar{M} := M+m$, then by the smoothness of $\bar{U}_{\mu}$ (Lemma~\ref{lemma:closeness-of-smoothing}): 
$$
\lv \nabla \bar{U}_{\mu}(\vx) -  \bar{U}_{\mu}(\bar{\vx}_{\mu}^*)\rv_2 = \lv \nabla \bar{U}_{\mu}(\vx) \rv_2 \le \bar{M}\lv \vx -\bar{\vx}_{\mu}^* \rv_2 \le\bar{M} \lv \vx \rv_2 + \bar{M} \lv \bar{\vx}_{\mu}^* \rv_2,
$$ 
where $\bar{\vx}_{\mu}^*$ is the (unique) minimizer of $\bar{U}_{\mu}$. Further, by \cite[Proposition 1]{durmus2016high} (see Theorem \ref{thm:durmus-1} in Appendix \ref{app:auxiliary}) we have that $\mathbb{E}_{x\sim \bar{p}^*_{\mu}}\left[\lv \vx  \rv_2^2\right]\le \frac{2d}{\lambda} + 2\lv \vx^*\rv_2^2$. Using these facts it is also possible to bound the second moment of $\bar{p}_K$. Consider random variables $\vy \sim \bar{p}_K$ and $\vx \sim \bar{p}_{\mu}^*$, such that $\vx$ and $\vy$ are optimally coupled; that is, $\mathbb{E}\left[\lv \vx -\vy \rv_2^2\right] = W_2^2(\bar{p}_K,\bar{p}^*_{\mu})$. Then, using Young's inequality, we have,
 \begin{align*}
     \mathbb{E}\left[\lv \vy\rv_2^2\right] & = \mathbb{E}\left[\lv \vy - \vx + \vx\rv_2^2\right] \le 2\mathbb{E}\left[\lv \vy- \vx\rv_2^2\right] + 2\mathbb{E}\left[\lv \vx\rv_2^2\right]  = \frac{4d}{\lambda} + 4\lv \vx^*\rv_2^2 +2W_2^2(\bar{p}_K,\bar{p}_{\mu^*}).
 \end{align*}
 Thus, by invoking~\cite[Proposition 1]{polyanskiy2016wasserstein} (see Proposition~\ref{prop:wassersteinstability} in Appendix~\ref{app:auxiliary}), we get:
 \begin{align}
    \nonumber \KL(& \bar{p}_K\lvert \bar{p}_{\mu}^*)\\
     &\le \bigg(\frac{\bar{M}\sqrt{\frac{2d}{\lambda} +2\lv \vx^* \rv_2^2 }}{2} + \frac{\bar{M}\sqrt{\frac{4d}{\lambda}+ 4\lv \vx^*\rv_2^2 + 2W_2^2(\bar{p}_K,\bar{p}_{\mu^*})}}{2} + \bar{M}\lv \vx^* \rv_2\bigg)W_2(\bar{p}_K,\bar{p}^*_{\mu}). \label{eq:KLboundcomposite}
 \end{align}
Finally, by Pinsker's inequality, we have:
\begin{align}\label{eq:TVcontraction-m-lmc-composite}
    \lv \bar{p}_K & - \bar{p}^*_{\mu} \rv_{\TV}\notag\\
    &\le \sqrt{\bigg(\frac{\bar{M}\sqrt{\frac{2d}{\lambda} +2\lv \vx^* \rv_2^2 }}{4} + \frac{\bar{M}\sqrt{\frac{4d}{\lambda}+ 4\lv \vx^*\rv_2^2 + 2W_2^2(\bar{p}_K,\bar{p}_{\mu^*})}}{4} + \frac{\bar{M}\lv \vx^* \rv_2}{2}\bigg)W_2(\bar{p}_K,\bar{p}^*_{\mu})},
\end{align}
where $W_2(\bar{p}_K,\bar{p}^*_{\mu})$ was bounded above in Eq.~\eqref{eq:contractionofwasserstein-m-lmc-good-composite}. 
By Lemma~\ref{lemma:wassersteincontrol-nonsmooth} (see Appendix~\ref{sec:omitted-pfs-main}), we have that:
\begin{align} \label{eq:TVapproxerror-composite}
\lv \bar{p}^* - \bar{p}^*_{\mu}\rv_{\TV} \le \frac{L \mu^{1+\alpha}d^{(1+\alpha)/2}}{1+\alpha} +\frac{ \lambda\mu^2 d}{2} .
\end{align}
Combining Eqs.~\eqref{eq:TVtriangle-composite},~\eqref{eq:contractionofwasserstein-m-lmc-good-composite},~\eqref{eq:TVcontraction-m-lmc-composite}  and \eqref{eq:TVapproxerror-composite} yields the first claim. 

For the remaining claim, we first choose $\mu$ so that $\lv \bar{p}^* - \bar{p}^*_{\mu}\rv_{\TV} \le \frac{\varepsilon}{2}.$ It is not hard to verify that the following choice suffices, as $m \geq \lambda$ (smoothness is always at least as high as the strong convexity of a function, and $m$ and $\lambda$ parameters come from the same function $\psi$):
\begin{align*}
   \mu = \min\bigg\{ \frac{\varepsilon^{\frac{1}{1+\alpha}}}{4 \max\{1, L^{\frac{1}{1+\alpha}}\}d^{1/2}},\; \sqrt{\frac{\varepsilon\lambda}{2 m^2 d}} \bigg\}.
\end{align*}
It remains to bound $\|\bar{p}_K - \bar{p}_{\mu}^*\|_{\TV}$ by $\varepsilon/2.$ To do so, we first bound $W_2(\bar{p}_K, \bar{p}_{\mu}^*).$ 

To simplify the upper bound on $W_2(\bar{p}_K, \bar{p}_{\mu}^*)$ from~\eqref{eq:contractionofwasserstein-m-lmc-good-composite}, we first show that under our choice of $\mu,$ $\sigma \le \left(2(M+m)\right)^{1/2}.$ Indeed, as $\sigma, M$ and $m$ are all non-negative, we have that it suffices to show that $\sigma^2 \leq 2(M+m).$ Recall from Lemma~\ref{lemma:variancebound} that:
$$
\sigma^2 \leq \frac{4 \mu^{2\alpha} L^2}{d^{1-\alpha}} + 4\mu^2 m^2. 
$$
By the choice of $\mu,$ as $\epsilon \leq 1$ and $d \geq 1,$ we have that $4\mu^2 m^2 \leq 2 m.$ Hence, to prove the claim, it remains to show that $\frac{4 \mu^{2\alpha} L^2}{d^{1-\alpha}} \leq 2M.$ Recalling that $M = \frac{L d^{\frac{1-\alpha}{2}}}{\mu^{1-\alpha}(1+\alpha)^{1-\alpha}} \geq \frac{L d^{\frac{1-\alpha}{2}}}{\sqrt{2}\mu^{1-\alpha}},$ and using elementary algebra and the choice of $\mu,$ the claim follows.

Hence, we have $W_2(\bar{p}_K,\bar{p}^*_{\mu}) \le \left(1-\lambda \eta\right)^{K/2} W_2(\bar{p}_0,\bar{p}^*_{\mu}) + 2\left(\frac{2(M+m)}{\lambda}\eta d\right)^{1/2}.$ Choosing:
\begin{align*}
    \eta & \leq \frac{\bar{\varepsilon}^2 \lambda}{64 d (M+m)} \quad\text{ and }\quad K \geq \frac{\log(2W_2(\bar{p_0}, \bar{p}_{\mu}^*)/\bar{\varepsilon})}{\lambda \eta},
\end{align*}
for some $\bar{\epsilon},$ ensures $W_2(\bar{p}_K,\bar{p}^*_{\mu}) \leq \bar{\varepsilon}.$ It only remains to choose $\bar{\varepsilon}$ so that $\|\bar{p}_K - \bar{p}_{\mu}^*\|_{\TV},$ which was bounded in Eq.~\eqref{eq:TVcontraction-m-lmc-composite}, is at most $\varepsilon/2$. 
Choosing:
\begin{align*}
    \bar{\varepsilon} = \frac{\varepsilon^2}{4 \max\{(M+m) (\sqrt{2d/\lambda + 2\|\vx^*\|_2^2} + 2\|\vx^*\|_2^2), 1\}}
\end{align*}
suffices, which gives the choice of parameters from the statement of the theorem, completing the proof.
%
\end{proof}
\begin{rem}\label{remark:stochastic-TV}
An interesting byproduct of the sequence of inequalities used in the proof of Theorem~\ref{thm:TVcompositecontract} is that they lead to bounds in total variation distance for~\eqref{eq:discrete-langevin} with stochastic gradients. This follows by combining the result from~\cite[Theorem~21]{durmus2019analysis} (see Theorem~\ref{thm:dalalyan-karagulyan} in Appendix~\ref{app:auxiliary}) with the inequality from Eq.~\eqref{eq:TVcontraction-m-lmc-composite}. Combining the two, we have that under the assumptions of Theorem~\ref{thm:dalalyan-karagulyan}:
\begin{align}\notag
    \lv \bar{p}_K & - \bar{p}^* \rv_{\TV}\notag\\
    &\le \sqrt{\bigg(\frac{{M}\sqrt{\frac{2d}{\lambda} +2\lv \vx^* \rv_2^2 }}{4} + \frac{{M}\sqrt{\frac{4d}{\lambda}+ 4\lv \vx^*\rv_2^2 + 2W_2^2(\bar{p}_K, \bar{p}^*)}}{4} + \frac{{M}\lv \vx^* \rv_2}{2}\bigg)W_2(\bar{p}_K, \bar{p}^*)},
\end{align}
where $W_2(\bar{p}_K, \bar{p}^*)$ is bounded as in Theorem~\ref{thm:dalalyan-karagulyan}. Thus, treating $M, \lambda,$ and $\|\vx^*\|_2$ as constants,~\eqref{eq:discrete-langevin} with stochastic gradients takes at most as many iterations to converge to  $\lv \bar{p}_K  - \bar{p}^* \rv_{\TV}\leq {\cal}{\varepsilon}$ as it takes to converge to $W_2(\bar{p}_K, \bar{p}^*)\leq \varepsilon^2.$ 

Note that the inequality from Eq.~\eqref{eq:TVcontraction-m-lmc-composite} is also precisely the reason the mixing time we obtain for~\eqref{eq:modifed-LMC} with $\alpha = 1$ (smooth potential) in total variation distance is quadratically higher than for the 2-Wasserstein distance. If this inequality improved, our bound would improve as well. 

Finally, we note that an obvious obstacle to carrying out the analysis directly in the total variation distance using the coupling technique (see Appendix~\ref{app:det-approx}) is the application of Girsanov's formula (see the proof of Lemma~\ref{lem:discretebound}). Specifically, when applying Girsanov's formula, we need to bound 
$$
\mathbb{E}\|\nabla \bar{U}_{\mu}(\tilde{\vx}_s) - \nabla \bar{U}(\tilde{\vx}_{k \eta})\|_2^2 = \mathbb{E}\|\nabla \bar{U}_{\mu}(\tilde{\vx}_s) - \nabla \bar{U}_{\mu}(\tilde{\vx}_{k \eta}) + \nabla \bar{U}_{\mu}(\tilde{\vx}_{k \eta}) - \nabla \bar{U}(\tilde{\vx}_{k \eta})\|_2^2.
$$
Applying Young's inequality, we need bounds on $\mathbb{E}\|\nabla \bar{U}_{\mu}(\tilde{\vx}_s) - \nabla \bar{U}_{\mu}(\tilde{\vx}_{k \eta})\|_2^2$ and $\mathbb{E}\| \nabla \bar{U}_{\mu}(\tilde{\vx}_{k \eta}) - \nabla \bar{U}(\tilde{\vx}_{k \eta})\|_2^2$. While bounding the former is not an issue, the latter can only be bounded using the variance bound from Lemma~\ref{lemma:variancebound}. Unfortunately, when $\alpha = 0,$ the bound on the variance is at least  $\frac{4 L^2}{d}$ (a constant independent of the step size), which leads to the similar blow up in the bound on the mixing time as in Corollary~\ref{cor:det-TV-bound}.
\end{rem}

\begin{lemma} \label{lemma:wassersteincontrol-nonsmooth} Let  $\bar{p}^*,$ and $\bar{p}_{\mu}^*$ be the distributions corresponding to the potentials $\bar{U},$ and $\bar{U}_{\mu},$ respectively. Then, we have:
\begin{align*}
 \lv \bar{p}^* - \bar{p}^*_{\mu} \rv_{\TV} & \le \frac{L \mu^{1+\alpha}d^{(1+\alpha)/2}}{1+\alpha} +\frac{ \lambda\mu^2 d}{2}.
\end{align*}
\end{lemma}

\begin{proof}
We can control the Kullback-Leibler divergence between the distributions by using \cite[Lemma 3]{dalalyan2017theoretical} (see Lemma~\ref{lemma:dalalyan-bnded-dist} in Appendix~\ref{app:auxiliary}). Using Lemma~\ref{lemma:closeness-of-smoothing}, we have  $0 \leq \bar{U}_{\mu} - \bar{U} \leq \frac{L \mu^{1+\alpha}d^{(1+\alpha)/2}}{1+\alpha} + \frac{\lambda\mu^2 d}{2}$. Therefore:
\begin{align}
    \KL(\bar{p}^*\lvert \bar{p}^*_{\mu}) &\le \frac{1}{2}\int \left(\bar{U}(\vx)-\bar{U}_{\mu}(\vx)\right)^2 \bar{p}^*(\vx)\dd\vx\notag\\
    &\leq \Big(\frac{L \mu^{1+\alpha}d^{(1+\alpha)/2}}{1+\alpha} + \frac{\lambda\mu^2 d}{2}\Big)^2 
\end{align}
Invoking Pinsker's inequality: $\lv \bar{p}^* - \bar{p}^*_{\mu}\rv_{\TV} \le \sqrt{\KL(p^* \lvert \bar{p}^*_{\mu})/2}$ yields the claim.
\end{proof}

\end{document}